\documentclass[twoside,11pt]{article}

\usepackage{jmlr2e}
\usepackage[T1]{fontenc}
\usepackage{graphicx}
\usepackage{chngcntr}
\usepackage{apptools}
\usepackage{amsmath}
\usepackage{amsfonts}
\usepackage{amstext}
\usepackage{amssymb}
\usepackage{float}
\usepackage{color}
\usepackage{graphicx}
\usepackage{enumerate}
\usepackage{float}
\usepackage{subfig}
\usepackage{setspace}
\usepackage{multirow}
\usepackage{nicefrac}
\usepackage[ruled,vlined,linesnumbered,resetcount,algosection]{algorithm2e}

\newtheorem{assumption}{Assumption}[section]

\numberwithin{equation}{section}
\numberwithin{figure}{section}

\newcommand{\R}{\mathbb{R}}
\newcommand{\E}{\mathbb{E}}

\newcommand{\PP}{\mathbb{P}}
\newcommand{\Fcal}{\mathcal{F}}
\newcommand{\Lcal}{\mathcal{L}}
\newcommand{\Ccal}{\mathcal{C}}
\newcommand{\Lcalloc}{\mathcal{L}_{\text{loc}}}
\newcommand{\Bcal}{\mathcal{B}}

\DeclareMathOperator*{\diag}{diag}

\newcommand{\Xtil}{\widetilde{X}}
\newcommand{\Dtil}{\widetilde{\Delta}}
\newcommand{\tr}{\mathrm{Tr}}

\ShortHeadings{Stochastic Modified Equations I: Mathematical Foundations}{Li, Tai and E}
\firstpageno{1}

\begin{document}

\title{Stochastic Modified Equations and Dynamics of Stochastic Gradient Algorithms I: Mathematical Foundations}

\author{\name Qianxiao Li \email liqix@ihpc.a-star.edu.sg \\
       \addr Institute of High Performance Computing\\
       Agency for Science, Technology and Research\\
       1 Fusionopolis Way, Connexis North, Singapore 138632
       \AND
       \name Cheng Tai \email chengtai@pku.edu.cn \\
       \addr Beijing Institute of Big Data Research\\
       and Peking University\\
       Beijing, China, 100080
       \AND
       \name Weinan E \email weinan@math.princeton.edu \\
	   \addr Princeton University\\
       Princeton, NJ 08544, USA \\
	   Beijing Institute of Big Data Research\\
       and Peking University, Beijing, China}

\editor{}

\maketitle

\begin{abstract}%
	We develop the mathematical foundations of the stochastic modified equations (SME) framework for analyzing the dynamics of stochastic gradient algorithms, where the latter is approximated by a class of stochastic differential equations with small noise parameters. We prove that this approximation can be understood mathematically as an weak approximation, which leads to a number of precise and useful results on the approximations of stochastic gradient descent (SGD), momentum SGD and stochastic Nesterov's accelerated gradient method in the general setting of stochastic objectives. We also demonstrate through explicit calculations that this continuous-time approach can uncover important analytical insights into the stochastic gradient algorithms under consideration that may not be easy to obtain in a purely discrete-time setting.
\end{abstract}


\section{Introduction}
\label{sec:introduction}
Stochastic gradient algorithms (SGA) are often used to solve optimization problems of the form
\begin{align}
	\min_{x\in \R^d}\quad  f(x) := \E f_\gamma(x)
	\label{eq:opt_problem}
\end{align}
where $\{ f_r : r \in \Gamma \}$ is a family of functions from $\R^d$ to $\R$ and $\gamma$ is a $\Gamma$-valued random variable, with respect to which the expectation is taken (these notions will be made precise in the following sections). For empirical loss minimization in supervised learning applications, $\gamma$ is usually a uniform random variable taking values in $\Gamma = \{1, 2, \dots, n \}$. In this case, $f$ is the total empirical loss function and $f_r$, $r\in \Gamma$ are the loss function due to the $r^\text{th}$ training sample. In this paper, we shall consider the general situation of a expectation over arbitrary index sets and distributions.

Solving~\eqref{eq:opt_problem} using the standard gradient descent (GD) on $x$ gives the iteration scheme
\begin{align}
	x_{k+1} = x_{k} - \eta \nabla \E f_{\gamma}(x_k), \label{eq:gd_iter}
\end{align}
for $k\geq 0$ and $\eta$ is a small positive step-size known as the learning rate.
Note that this requires the evaluation of the gradient of an expectation, which can be costly (in this empirical risk minimization case, this happens when $n$ is large).
In its simplest form, the stochastic gradient descent (SGD) algorithm replaces the expectation of the gradient with a sampled gradient, i.e.
\begin{align}
	x_{k+1} = x_{k} - \eta \nabla f_{\gamma_k}(x_k),
	\label{eq:sga_iter}
\end{align}
where each $\gamma_k$ is an independent and identically distributed (i.i.d.) random variable with the same distribution as $\gamma$. Under mild conditions, we then have $\E [ \nabla f_{\gamma_k}(x_k) | x_k ] = \nabla \E f(x_k)$. In other words,~\eqref{eq:sga_iter} is a sampled version of~\eqref{eq:gd_iter}.

In the literature, many convergence results are available for SGD and its variants~\citep{shamir2013stochastic,moulines2011non,needell2014stochastic,xiao2014proximal,shalev2014accelerated,bach2013non,defossez2015averaged}. However, it is often the case that different analysis techniques must be adopted for different variants of the algorithms and there generally lacked a systematic approach to study their precise dynamical properties. In~\cite{li2015dynamics}, a general approach was introduced to address this problem, in which discrete-time stochastic gradient algorithms are approximated by continuous-time stochastic differential equations with the noise term depending on a small parameter (the learning rate). This can be viewed as a generalization of the method of modified equations~\citep{hirt1968heuristic,noh1960difference,daly1963stability,warming1974modified} to the stochastic setting, and allows one to employ tools from stochastic calculus to systematically analyze the dynamics of stochastic gradient algorithms. The stochastic modified equations (SME) approach was further developed in~\cite{li2017stochastic}, where a weak approximation result for the SGD was proved in a finite-sum-objective setting.

The present series of papers builds on the earlier work of~\cite{li2015dynamics,li2017stochastic} and aims to establish the framework of stochastic modified equations and their applications in greater generality and depth, and highlight the advantages of this systematic framework for studying stochastic gradient algorithms using continuous-time methods. As the first in the series, this paper will focus on mathematical aspects, namely the main approximation theorems relating stochastic gradient algorithms to stochastic modified equations in the form of weak approximations. These generalize the approximation results in~\cite{li2017stochastic} in various aspects. In a subsequent paper in the series, we will discuss the application of this formalism to adaptive stochastic gradient algorithms and related problems.

The organization of this paper is as follows. We first discuss related work in Sec.~\ref{sec:related_work}, especially in the context of continuous-time approximations. Next, we motivate the SME approach and set up the precise mathematical framework in Sec.~\ref{sec:math_framework}. We then prove in Sec.~\ref{sec:approx_theorems} a central result relating discrete stochastic algorithms and continuous stochastic processes, which allows us to derive SMEs for stochastic gradient descent and variants. In Sec.~\ref{sec:applications}, the SME approach is used to analyze the dynamics of stochastic gradient algorithms when applied to optimize a simple yet non-trivial objective. Lastly, we conclude with some discussion of our results in Sec.~\ref{sec:conclusion}. The longer proofs of the results used in the paper are organized in the appendix. These are essentially self-contained, but basic knowledge of stochastic calculus and probability theory are assumed. Unfamiliar readers may refer to standard introductory texts, such as~\cite{durrett2010probability} and~\cite{oksendal2013stochastic}.

\subsection{Notation}
In this paper, we adhere wherever possible to the following notation. Dimensional indices are written as subscripts with a bracket to avoid confusion with other sequential indices (e.g.\,time, iteration number), which do not have brackets. When more than one indices are present, we separate them with a comma, e.g. $x_{k,(i)}$ is the $i$-th coordinate of the vector $x_k$, the $k^{\text{th}}$ member of a sequence.
We adopt the Einstein's summation convention, where repeated (spatial) indices are summed, i.e. $x_{(i)}x_{(i)} := \sum_{i=1}^d x_{(i)} x_{(i)}$. For a matrix $A$, we denote by $\lambda(A)=\{ \lambda_1(A), \lambda_2(A), \dots \}$ the set of eigenvalues of $A$. If $A$ is Hermitian, then the eigenvalues are ordered so that $\lambda_1(A)$ denotes a maximum eigenvalue. We denote the usual Euclidean norm by $|\cdot|$ and for higher rank tensors, we use the same notation to denote the flattened vector norms (e.g.\,for matrices it will be the Frobenius norm). The $\wedge$ symbols denotes the minimum operator, i.e. $a \wedge b := \min(a,b)$.

For a probability space (or generally, a measure space) $(\Omega,\Fcal,\PP)$, the symbol $\Lcal(\Omega,\Fcal,\PP)$, $p\in(1,\infty)$ denotes the usual Lebesgue spaces, i.e. $u\in\mathcal{L}^p(\Omega,\Fcal,\PP)$ if
\begin{align*}
	\| u \|_{\Lcal^p(\Omega,\Fcal,\PP)}^p := \int_{\Omega} |u(\omega)|^p d\PP(\omega)
	\equiv \E |u|^p < \infty.
\end{align*}
When the underlying probability space is obvious, we use the shorthand $\Lcal^p(\Omega) \equiv \Lcal(\Omega,\Fcal,\PP)$. In addition, when $\Omega=\R^d$, we also write the local $\Lcal^p$ spaces as $\Lcal^p_{\text{loc}}(\R^d)$, which contains $u$ for which $|u|^p$ is integrable on compact subsets of $\R^d$.

Finally, we note that in the proofs of various results, we typically use the letter $C$ (whose value may change across results) to denote a generic positive constant. This is usually independent of the learning rate $\eta$, but if not explicitly stated otherwise, it may depend on e.g.\,Lipschitz constants, ambient dimensions, etc.

\section{Related work}
\label{sec:related_work}
In this section, we discuss several related works on analyzing discrete-time algorithms using continuous-time approaches. The idea of approximating discrete-time stochastic algorithms by continuous equations dates back to the large body of work known as stochastic approximation theory~\citep{kushner2003stochastic,ljung2012stochastic}.
These typically establish law of large numbers type results where the limiting equation is an ODE, which can then be used to prove powerful convergence results for the stochastic algorithms under consideration. A notion of convergence in distribution, similar to a central limit theorem, was also studied for the purpose of estimating the rate of convergence of the ODE methods~\citep{kushner1978rates,kushner1984invariant,kushner2012stochastic}, where connections between leading order perturbations and Ornstein-Uhlenbeck (OU) processes are established. However, these estimates are not systematically used to systematically study the dynamics of stochastic gradient algorithms.

As far as the authors are aware, the first work on using stochastic differential equations to study the precise properties of stochastic gradient algorithms are the independent works of~\cite{li2015dynamics} and~\cite{mandt2015continuous}. In~\cite{li2015dynamics}, a systematic framework of SDE approximation of SGD and SGD with momentum are derived and applied to study dynamical properties of the stochastic algorithms as well as adaptive parameter tuning schemes. These go beyond OU process approximations and this distinction is important since the OU process is not always the appropriate stochastic approximation in general settings (See Sec.~\ref{sec:sme_sgd} of this paper).
In~\cite{mandt2015continuous}, a similar procedure is employed to derive a SDE approximation for the SGD, from which issues such as choice of learning rates are studied. Although the concrete analysis in~\cite{mandt2015continuous} is on the restricted case of constant diffusion matrices leading to OU processes, the essential ideas on the general leading order approximation are also discussed.

It is important to note that the approximation arguments in both~\cite{li2015dynamics} and~\cite{mandt2015continuous} are heuristic from a mathematical point of view. In~\cite{li2017stochastic}, the SME approximation is rigorously proved in the finite-sum-objective case with strong regularity conditions, and further asymptotic analysis and tuning algorithms are studied.
The SME approach has subsequently been utilized to study variants of stochastic gradient algorithms, including those in the distributed optimization setting~\citep{an2018stochastic}.
The work of~\cite{mandt2015continuous} is further developed in~\cite{mandt2016variational,mandt2017stochastic}, with applications such as the development scalable MCMC algorithms.

The present paper builds on the earlier work of~\cite{li2015dynamics,li2017stochastic}, but focuses on extending and solidifying the mathematical aspects. In particular, we present an entirely rigorous and self-contained mathematical formulation of the SME framework that applies to more general algorithms (including momentum SGD and stochastic Nesterov's accelerated gradient method) and more general objectives (expectation over random functions, instead of just a finite-sum). Moreover, various regularity conditions in~\cite{li2017stochastic} have been relaxed. The main approximation procedure is inspired by the seminal works of~\cite{mil1986weak,mil1975approximate} in numerical analysis of stochastic differential equations, but lower regularity conditions are required in our case due to the presence of the small noise parameter, which allows for better truncation of It\^{o}-Taylor expansions.
The mathematical analysis of the SME-type approximation for the SGD was also performed in~\cite{feng2017note,hu2017diffusion} using semi-group approaches, although the smoothness requirements presented there are greater than those established using the current methods. Lastly, the Nesterov's accelerated gradient SME we derive in Sec.~\ref{sec:sme_snag} can be viewed as a generalization of the ODE approach in~\cite{su2014differential} to stochastic gradients, and we show that the presence of noise gives additional features to the dynamics. Finally, we note that continuous-time approximations that establish links between optimization, calculus of variations and symplectic integration has been studied in~\cite{wibisono2016variational,betancourt2018symplectic}.

\section{Stochastic modified equations}
\label{sec:sme}
We now introduce the stochastic modified equations framework. The starting motivation is the observation that GD iterations is a (Euler) discretization of the continuous-time, ordinary differential equation
\begin{align}
	\tfrac{dx}{dt} = -\nabla f(x),
	\label{eq:gradflow}
\end{align}
and studying~\eqref{eq:gradflow} can give us important insights to the dynamics of the discrete-time algorithm for small enough learning rates. The natural question when extending this to SGD is, what is the right continuous-time equation to consider? Below, we begin with some heuristic considerations.

\subsection{Heuristic motivations}
\label{sec:heuristics}

we rewrite the SGD iteration~\eqref{eq:sga_iter} as
\begin{align}
x_{k+1} = x_{k} - \eta \nabla f(x_k) + \sqrt{\eta} V_k(x_k, \gamma_k),
\label{eq:sga_rewrite}
\end{align}
where $V_k(x_k, \gamma_k) = \sqrt{\eta}(\nabla f(x_k) - \nabla f_{\gamma_k} (x_k))$ is a $d$-dimensional random vector. A straightforward calculation shows that
\begin{align}
\label{eq:sga_cov}
	& \E[V_k | x_k] = 0 \nonumber \\
	&\text{cov}[V_{k},V_{k} | x_k] = \eta \Sigma(x_k), \nonumber \\
	&\Sigma(x_k) := \E [ (\nabla f_{\gamma_k}(x_k) - \nabla f(x_k)) {(f_{\gamma_k}(x_k) - \nabla f(x_k))}^T | x_k],
\end{align}
i.e.\,conditional on $x_k$, $V_k(x_k)$ has $0$ mean and covariance $\eta\Sigma(x_k)$.
Here, $\Sigma$ is simply the conditional covariance of the stochastic gradient approximation $\nabla f_{\gamma} $ of $\nabla f$.

Now, consider a time-homogeneous It\^{o} stochastic differential equation (SDE) of the form
\begin{align}
\label{eq:sde_0}
	dX_t = b(X_t) dt + \sqrt{\eta} \sigma(X_t)dW_t,
\end{align}
where $X_t\in \mathbb{R}^d$ for $t\geq0$ and $W_t$ is a standard $d$-dimensional Wiener process. The function $b:\mathbb{R}^d\rightarrow\mathbb{R}^d$ is known as the drift and $\sigma:\mathbb{R}^d\rightarrow\mathbb{R}^{d\times d}$ is the diffusion matrix.
The key observation is that if we apply the Euler discretization with step-size $\eta$ to~\eqref{eq:sde_0}, approximating $X_{k\eta}$ by $\hat{X}_k$, we obtain the following discrete iteration for the latter:
\begin{align}
	\hat{X}_{k+1} = \hat{X}_k + \eta b(\hat{X}_k) + \eta \sigma(\hat{X}_k) Z_k,
\end{align}
where $Z_k := W_{(k+1)\eta} - W_{k\eta}$ are $d$-dimensional i.i.d.\,standard normal random variables.
Comparing with~\eqref{eq:sga_rewrite}, if we set $b = -\nabla f$, $\sigma(x) = {\Sigma(x)}^{\nicefrac{1}{2}}$ and identify $t$ with $k\eta$, we then have matching first and second conditional moments. Hence, this motivates the approximating equation
\begin{align}
	dX_t = - \nabla f(X_t) dt +  {(\eta\Sigma(X_t))}^{1/2} dW_t.
	\label{eq:motivating_sde}
\end{align}
Note that as this heuristic argument shows, the presence of the small parameter $\sqrt{\eta}$ on the diffusion term is necessary to model the fact that when learning rate decreases, the fluctuations to the SGA iterates must also decrease.

The immediate mathematical question is then: in what sense is an SDE like~\eqref{eq:motivating_sde} an approximation of~\eqref{eq:sga_iter}? Let us now establish the precise mathematical framework in which we can answer this question.

\subsection{The mathematical framework}
\label{sec:math_framework}

Let $(\Omega, \Fcal, \PP)$ be a sufficiently rich probability space and $(\Gamma, \Fcal_\Gamma)$ be a measure space representing the index space for our random objectives. Let $\gamma : \Omega \rightarrow \Gamma$ be a random variable and $(r, x) \mapsto f_r(x)$ a measurable mapping from $\Gamma \times \R^d$ to $\R$. Hence, for each $x$, $f_{\gamma}(x)$ is a random variable. Throughout this paper, we assume the follow facts about $f_{\gamma}(x)$:
\begin{assumption}
\label{assu:f_gamma}
	The random variable $f_{\gamma}(x)$ satisfies
	\begin{enumerate}[(i)]
		\item $f_{\gamma}(x) \in \Lcal^1(\Omega)$ for all $x\in \R^d$
		\item $f_{\gamma}(x)$ is continuously differentiable in $x$ almost surely and
		for each $R>0$, there exists a random variable $M_{R,\gamma}$ such that $\max_{|x| \leq R} | \nabla f_{\gamma} (x) | \leq M_{R,\gamma}$ almost surely, with $\E |M_{R,\gamma}| < \infty$
		\item $\nabla f_{\gamma}(x)\in \Lcal^2(\Omega)$ for all $x\in \R^d$
	\end{enumerate}
\end{assumption}
Note that in the empirical risk minimization case where $\Gamma$ is finite, the conditions above are often trivially satisfied.
Condition (i) in Assumption~\ref{assu:f_gamma} allows us to define the total objective function we would like to minimize as the expectation
\begin{align}
	f(x) := \E f_{\gamma}(x) \equiv \int_{\Omega} f_{\gamma(\omega)} (x) d\PP(\omega).
\end{align}
Moreover, Assumption~\ref{assu:f_gamma} (ii) implies via the dominated convergence theorem that $\E \nabla f_\gamma = \nabla \E f_\gamma \equiv \nabla f$.
Now, let $\{\gamma_{k} : k=0, 1, \dots \}$ be a sequence of i.i.d.\,$\Gamma$-valued random variables with the same distribution as $\gamma$. Let $x_0\in\R^d$ be fixed and define the generalized stochastic gradient iteration as the stochastic process
\begin{align}
	x_{k+1} = x_{k} + \eta h(x_k, \gamma_k, \eta)
	\label{eq:sga_generalized}
\end{align}
for $k \geq 0$, where $h:\R^d \times \Gamma \times \R \rightarrow \R^d$ is a measurable function and $\eta>0$ is the learning rate. In the simple case of SGD, we have $h(x,r,\eta)= - \nabla f_r (x)$, but we shall consider the generalized version above so that modified equations for SGD variants can also be derived from our approximation theorems.

Next, let us define the class of approximating continuous stochastic processes, which we call stochastic modified equations. Consider the time-homogeneous It\^{o} diffusion process $\{ X_t : t\geq 0\}$ represented by the following stochastic differential equation (SDE)
\begin{align}
	dX_t = b(X_t,\eta) dt + \sqrt{\eta} \sigma(X_t,\eta) dW_t,\qquad X_0 = x_0
	\label{eq:sde_generalized}
\end{align}
where $\{ W_t : t\geq 0\}$ is a standard $d$-dimensional Wiener process independent of $\{\gamma_k\}$, $b:\R^d\times\R\rightarrow\R^d$ is the approximating drift vector and $\sigma: \R^d \times \R \rightarrow \R^{d\times d}$ is the approximating diffusion matrix. In the following, we will need to pick $b,\sigma$ appropriately so that~\eqref{eq:sga_generalized} is approximated by~\eqref{eq:sde_generalized}, the sense of which we now describe.

First, notice that the stochastic process $\{x_k\}$ induces a probability measure on the product space ${\R^d\times\R^d\times\cdots}$, whereas $\{ X_t \}$ induces a probability measure on $\Ccal^0([0,\infty),\R^d)$. Hence, we can only compare their values by sampling a discrete number of points from the latter.
Second, the process $\{x_k\}$ is adapted to the filtration generated by $\{\gamma_k \}$ (e.g.\,in the case of SGD, this is the random sampling of functions in $\{ f_r \}$), whereas the process $\{ X_t \}$ is adapted to an independent, Wiener filtration. Hence, it is not appropriate to compare individual sample paths. Rather, we define below a sense of weak approximations by comparing the distributions of the two processes.

\begin{definition}
\label{def:poly_growth}
	Let $G$ denote the set of continuous functions $\R^d\rightarrow\R$ of at most polynomial growth, i.e.\,$g \in G$ if there exists positive integers $\kappa_1,\kappa_2>0$ such that
	\begin{align*}
		| g(x) | \leq \kappa_1 (1+| x |^{2\kappa_2} ),
	\end{align*}
	for all $x\in\R^d$. Moreover, for each integer $\alpha\geq 1$ we denote by $G^\alpha$ the set of $\alpha$-times continuously differentiable functions $\R^d\rightarrow \R$ which, together with its partial derivatives up to and including order $\alpha$, belong to $G$. Note that each $G^\alpha$ is a subspace of $\Ccal^\alpha$, the usual space of $\alpha$-times continuously differentiable functions. Moreover, if $g$ depends on additional parameters, we say $g\in G^\alpha$ if the constants $\kappa_1,\kappa_2$ are independent of these parameters, i.e.\,$g\in G^\alpha$ uniformly. Finally, the definition generalizes to vector-valued functions coordinate-wise in the co-domain.
\end{definition}
\begin{definition}
\label{def:weak_conv}
	Let $T>0$, $\eta\in(0,1\wedge T)$, and $\alpha \geq 1$ be an integer. Set $N=\lfloor T/\eta \rfloor$.
	We say that a continuous-time stochastic process $\{X_t:t\in[0,T]\}$ is an order $\alpha$ weak approximation of a discrete stochastic process $\{x_k:k=0,\dots,N\}$ if for every $g \in G^{\alpha+1}$, there exists a positive constant $C$, independent of $\eta$, such that
	\begin{align}
	\label{eq:weak_conv}
		\max_{k=0,\dots,N} | \E g(x_{k}) - \E g (X_{k\eta}) | \leq C \eta^\alpha.
	\end{align}
\end{definition}
Let us discuss briefly the notion of \emph{weak approximation} as introduced above. These are approximations of the distribution of sample paths, instead of the sample paths themselves. This is enforced by requiring that the expectations of the two processes $\{ X_t \}$ and $\{ x_k \}$ over a sufficiently large class of test functions to be close. In our definition, the test function class $G^{\alpha+1}$ is quite large, and in particular it includes all polynomials. Thus, Eq.~\eqref{eq:weak_conv} implies in particular that all moments of the two processes become close at the rate of $\eta^\alpha$, and hence so must their distributions. The notion of weak approximation must be contrasted with that of \emph{strong approximations}, where one would for example require (in the case of~\emph{mean-square approximations})
\begin{align*}
	{[\E | x_k - X_{k\eta} |^2]}^{\nicefrac{1}{2}} \leq C\eta^{\alpha}.
\end{align*}
The above forces the actual sample-paths of the two processes to be close, per realization of the random process, which severely limits its application. In fact, one important advantage of weak approximations is that the approximating SDE process $X_t$ can in fact approximate discrete stochastic processes whose step-wise driving noise is not Gaussian, which is exactly what we need to analyze general stochastic gradient iterations.

\section{The approximation theorems}
\label{sec:approx_theorems}

We now present the main approximation theorems. The derivation is based on the following two-step process:
\begin{enumerate}
	\item We establish a connection between one-step approximation and approximation on a finite time interval.
	\item We construct a one-step approximation that is of order $\alpha+1$, and so the approximation on a finite interval is of order $\alpha$.
\end{enumerate}

\subsection{Relating one-step to $N$-step approximations}
\label{sec:one_to_n_step}

Let us consider generally the question of the relationship between one-step approximations and approximations on a finite interval. Let $T>0$, $\eta\in(0,1 \wedge T)$ and $N=\lfloor T/\eta \rfloor$ and recall the general SGA iterations
\begin{align}
\label{eq:milstein_discrete}
	x_{k+1} = x_k + \eta h(x_k, \gamma_k, \eta), \quad x_0\in\R^d,
	\quad k=0,\dots,N.
\end{align}
and the general candidate family of approximating SDEs
\begin{align}
\label{eq:milstein_cts}
	dX^{\eta,\epsilon}_t = b(X^{\eta,\epsilon}_t,\eta, \epsilon) dt
	+ \sqrt{\eta}\sigma(X^{\eta,\epsilon}_t,\eta, \epsilon) dW_t,
	\quad X_0 = x_0, \quad t\in[0,T],
\end{align}
where $\epsilon \in (0,1)$ is a mollification parameter, whose role will become apparent later. To reduce notational clutter and improve readability, unless some limiting procedure is considered, we shall not explicit write the dependence of $X^{\eta,\epsilon}_t$ on $\eta, \epsilon$ and simply denote by $X_t$ the solution of the above SDE.
Let us also denote for convenience $\Xtil_{k}:=X_{k\eta}$. Further, let $\{ X^{x,s}_t : t \geq s \}$ denote the stochastic process obeying the same equation~\eqref{eq:milstein_cts}, but with the initial condition $X^{x,s}_s = x$. We similarly write $\Xtil^{x,l}_{k} := X^{x,l\eta}_{k\eta}$ and denote by $\{ x^{x,l}_k : k \geq l \}$ the stochastic process satisfying~\eqref{eq:milstein_discrete} but with $x_l=x$.

Throughout this section, we assume the following conditions:
\begin{assumption}
	The functions $b: \R^d \times (0,1 \wedge T) \times (0,1) \rightarrow \R^d$ and $\sigma: \R^d \times (0,1 \wedge T) \times (0,1) \rightarrow \R^{d\times d}$ satisfy:
	\begin{enumerate}
		\item Uniform linear growth condition
		\begin{align*}
			| b(x, \eta, \epsilon) |^2 + | \sigma(x, \eta, \epsilon) |^2 \leq L^2 (1 + |x|^2)
		\end{align*}
		for all $x, y \in \R^d$, $\eta \in (0,1 \wedge T)$, $\epsilon \in (0,1)$.
		\item Uniform Lipschitz condition
		\begin{align*}
			| b(x, \eta, \epsilon) - b(y, \eta, \epsilon) | + | \sigma(x, \eta, \epsilon) - \sigma(y, \eta, \epsilon) | \leq L |x - y|
		\end{align*}
		for all $x, y \in \R^d$, $\eta \in (0,1 \wedge T)$, $\epsilon \in (0,1)$.
	\end{enumerate}
\end{assumption}
Note that 2 implies 1 if there is at least one $x$ where the supremum of $b,\sigma$ over $\eta,\epsilon$ is finite. In particular, these conditions imply via Thm.~\ref{app:thm:exist_uniq} that there exists a unique solution to Eq.~\ref{eq:milstein_cts}.

Now, let us denote the one-step changes
\begin{align}
\label{eq:Delta_defn}
	\Delta(x) := x^{x,0}_1 - x,
	\qquad\qquad
	\Dtil(x) := \Xtil^{x,0}_1 - x.
\end{align}
We prove the following result which relates one-step approximations with approximations on a finite time interval.

\begin{theorem}
\label{thm:one_step_to_multi_step}
	Let $T>0$, $\eta\in(0,1 \wedge T)$, $\epsilon\in(0,1)$ and $N=\lfloor T/\eta\rfloor$. Let $\alpha\geq 1$ be an integer.
	Suppose further that the following conditions hold:
	\begin{enumerate}[(i)]
		\item There exists a function $\rho:(0,1) \rightarrow \R_+$ and $K_1\in G$ independent of $\eta,\epsilon$ such that
		\begin{align*}
			\left|
				\E \prod_{j=1}^{s} \Delta_{(i_j)}(x) - \E \prod_{j=1}^s \Dtil_{(i_j)}(x)
			\right|
			\leq K_1(x)( \eta \rho(\epsilon) + \eta^{\alpha+1}),
		\end{align*}
		for $s=1,2,\dots,\alpha$ and
		\begin{align*}
			\E \prod_{j=1}^{\alpha+1}
			\left|
				\Delta_{(i_j)}(x)
			\right|
			\leq K_1(x)\eta^{\alpha+1},
		\end{align*}
		where $i_j \in \{ 1,\dots,d \}$.
		\item For each $m\geq 1$, the $2m$-moment of $x^{x,0}_k$ is uniformly bounded with respect to $k$ and $\eta$, i.e.\,there exists
		a $K_2 \in G$, independent of $\eta,k$, such that
		\begin{align*}
			\E | x^{x,0}_k |^{2m} \leq K_2(x),
		\end{align*}
		for all $k=0,\dots,N\equiv \lfloor T/\eta \rfloor$.
	\end{enumerate}
	Then, for each $g\in G^{\alpha+1}$, there exists a constant $C>0$, independent of $\eta,\epsilon$, such that
	\begin{align*}
	\max_{k=0,\dots,N}
	\left|
		\E g(x_k) - \E g(X_{k\eta})
	\right| \leq C ( \eta^\alpha + \rho(\epsilon) )
	\end{align*}
\end{theorem}
The proof of Thm.~\ref{thm:one_step_to_multi_step} requires a number of technical results that we defer to the appendix. Below, we demonstrate the main ingredients of the proof and refer to the appendix where the proofs of the auxiliary results are fully presented.
\begin{proof}
In this proof, since there are many conditioning on the initial condition, to prevent nested superscripts we shall introduce the alternative notation $X_t(x,s) \equiv X_t^{x,s}$, and similarly for $\tilde{X}_k$ and $x_k$.
Fix $g\in G^{\alpha+1}$ and $1 \leq k\leq N$. We have
\begin{align*}
	\E g(X_{k\eta}) = \E g(\Xtil_{k}) = \E g(\Xtil_{k}({\Xtil_1,1}))
	- \E g(\Xtil_{k}({x_1, 1}))
	+ \E g(\Xtil_{k}({x_1, 1})).
\end{align*}
% \begin{align*}
% 	\E g(X_{k\eta}) = \E g(\Xtil_{k}) = \E g(\Xtil^{\Xtil_1,1}_{k}) - \E g(\Xtil^{x_1, 1}_{k})
% 						  + \E g(\Xtil^{x_1, 1}_{k}).
% \end{align*}
If $k>1$, by noting that $\Xtil_{k}({x_1, 1}) = \Xtil_{k}({\Xtil_{2}({x_1,1}), 2})$, we get
% If $k>1$, by noting that $\Xtil^{x_1, 1}_{k} = \Xtil^{\Xtil^{x_1,1}_{2}, 2}_{k}$, we get
\begin{align*}
	\E g(\Xtil_{k}({x_1, 1})) = \E g(\Xtil_{k}({\Xtil_{2}({x_1,1}), 2}))
								  - \E g(\Xtil_{k}({x_2, 2}))
								  + \E g(\Xtil_{k}({x_2, 2}))
\end{align*}
Continuing this process, we then have
\begin{align*}
	\E g(\Xtil_{k}) =& \sum_{l=1}^{k-1}
		\E g(\Xtil_{k}({\Xtil_{l}({x_{l-1},l-1}), l}))
		- \E g(\Xtil_{k}({x_l,l}))  \nonumber \\
	&+ \E g(\Xtil_{k}({x_{k-1},k-1}))
\end{align*}
and hence by subtracting $\E g(x_k) \equiv \E g(x_k({x_{k-1}, k-1}))$ we get
\begin{align*}
	\E g(\Xtil_{k}) - \E g(x_k) &= \sum_{l=1}^{k-1}
		\E g(\Xtil_{k}({\Xtil_{l}({x_{l-1},l-1}), l}))
		- \E g(\Xtil_{k}({x_l,l}))
		\nonumber \\
	&+ \E g(\Xtil_{k}({x_{k-1},k-1})) - \E g(x_k({x_{k-1},k-1}))
\end{align*}
and so
\begin{align*}
	\E g(\Xtil_{k}) - \E g(x_k) &= \sum_{l=1}^{k-1}
		\E
		\E \left[
				g(\Xtil_{k}({\Xtil_{l}({x_{l-1},l-1}), l}))
				\Big|
				\Xtil_{l}({x_{l-1},l-1})
			\right]
		- \E
		\E \left[
				g(\Xtil_{k}({x_l,l}))
				\Big|
				x_l
			\right] \nonumber \\
	\nonumber \\
	&+ \E g(\Xtil_{k}({x_{k-1},k-1})) - \E g(x_k({x_{k-1},k-1})),
\end{align*}
Now, let $u(x,s) = \E g(X_{k\eta}({x,s}))$. Then, we have
\begin{align*}
	| \E g(\Xtil_{k}) - \E g(x_k) |
	&\leq \sum_{l=1}^{k-1}
		|
			\E u(\Xtil_{l}({x_{l-1},l-1}), l\eta)
			- \E u(x_{l}({x_{l-1},l-1}), l\eta)
		|
		\nonumber \\
	&+ | \E g(\Xtil_{k}({x_{k-1},k-1})) - \E g(x_k({x_{k-1},k-1})) |
	\nonumber \\
	&\leq \sum_{l=1}^{k-1}
		\E |
			\E [u(\Xtil_{l}({x_{l-1},l-1}), l\eta) | x_{l-1} ]
			- \E [u(x_{l}({x_{l-1},l-1}), l\eta) | x_{l-1} ]
		|
		\nonumber \\
	&+ \E |
		\E [g(\Xtil_{k}({x_{k-1},k-1})) | x_{k-1}]
		- \E [g(x_k({x_{k-1},k-1})) | x_{k-1}]
	|.
\end{align*}
Using Prop.~\ref{app:prop:u_poly_estimate}, $u(\cdot,s) \in G^{\alpha+1}$ uniformly in $s$, $t$, $\eta$ and $\epsilon$. Thus, by Assumption (i) and Lem.~\ref{app:lem:u_eta_estimate},
\begin{align*}
	| \E g(x_k) - \E g(\Xtil_{k}) |
	\leq& (\eta\rho(\epsilon) + \eta^{\alpha+1})
	\left(
		\sum_{l=1}^{k-1} \E K_{l-1}(x_{l-1}) + \E K_{k-1}(x_{k-1})
	\right) \\
	\leq& (\eta\rho(\epsilon) + \eta^{\alpha+1}) \sum_{l=0}^{N} \kappa_{l,1} (1 + \E | x_{l} |^{2\kappa_{l,2}}),
\end{align*}
where in the last line we used moment estimates from Thm.~\ref{app:thm:moment_estimates}.
Finally, using Assumption (ii) and the fact that $N\leq T/\eta$, we have
\begin{align*}
	| \E g(x_k) - \E g(X_{k\eta}) | =
	| \E g(x_k) - \E g(\Xtil_{k}) | \leq C (\rho(\epsilon) + \eta^{\alpha}).
\end{align*}
\end{proof}

\subsection{SME for stochastic gradient descent}
\label{sec:sme_sgd}

Thm.~\ref{thm:one_step_to_multi_step} allows us to prove the main approximation results for the current paper. In particular, in this section we derive a second-order accurate weak approximation for the simple SGD iterations~\eqref{eq:sga_iter}, from which a simpler, first-order accurate approximation also follows. As seen in Thm.~\ref{thm:one_step_to_multi_step}, we need only verify the conditions (i)-(ii) in order to prove the weak approximation result.
These conditions mostly involve moment estimates, which we now perform.
To simplify presentation, we introduce the following shorthand. Whenever we write
\begin{align*}
	\psi(x) = \psi_0(x)+\eta \psi_1(x) + \mathcal{O}(r(\eta, \epsilon)),
\end{align*}
for some remainder term $r(\eta,\epsilon)$, we mean: there exists $K \in G$ independent of $\eta,\epsilon$ such that
\begin{align*}
	| \psi(x) - \psi_0(x) - \eta \psi_1(x) | \leq K(x) r(\eta,\epsilon).
\end{align*}
Now, let us set in~\eqref{eq:milstein_cts}
\begin{align*}
	b(x,\eta,\epsilon) &= b_0(x, \epsilon) + \eta b_1(x, \epsilon) \\
	\sigma(x,\eta, \epsilon) &= \sigma_0(x, \epsilon),
\end{align*}
where $b_0,b_1,\sigma_0$ are functions to be determined. We have the following moment estimate.
\begin{lemma}
\label{lem:Dtil}
	Let $\Dtil(x)$ be defined as in~\eqref{eq:Delta_defn}. Suppose further that
	with $b_0,b_1,\sigma_0 \in G^3$. Then we have
	\begin{enumerate}[(i)]
		\item $\E \Dtil_{(i)}(x)
		= {b_0(x,\epsilon)}_{(i)}\eta +
		[
			\tfrac{1}{2}{b_0(x,\epsilon)}_{(j)}\partial_{(j)} {b_0(x,\epsilon)}_{(i)}
			+ {b_1(x,\epsilon)}_{(i)}
		]
		\eta^2 + \mathcal{O}(\eta^3)$,
		\item $\E \Dtil_{(i)}(x)\Dtil_{(j)}(x)
		= [{b_0(x,\epsilon)}_{(i)}{b_0(x,\epsilon)}_{(j)}
		+ {\sigma_0(x,\epsilon)}_{(i,k)} {\sigma_0(x,\epsilon)}_{(j,k)} ]\eta^2 + \mathcal{O}(\eta^3)$,
		\item $\E \prod_{j=1}^3
		| \Dtil_{(i_j)}(x) |
		= \mathcal{O}(\eta^3)$.
	\end{enumerate}
\end{lemma}
\begin{proof}
	To obtain (i)-(iii), we simply apply Lem.~\ref{app:lem:ito_taylor} with $\psi(z) = \prod_{j=1}^{s} (z_{(i_j)} - x_{(i_j)})$ for $s=1,2,3$ respectively.
\end{proof}

Next, we estimate the moments of the SGA iterations below.
\begin{lemma}
\label{lem:D}
	Let $\Delta(x)$ be defined as in~\eqref{eq:Delta_defn} with the SGD iterations, i.e.\,$h(x,r,\eta)=-\nabla f_r(x)$. Suppose that for each $x\in\R^d$, $f \in G^1$. Then,
	\begin{enumerate}[(i)]
		\item $\E \Delta_{(i)}(x) = -\partial_{(i)} f(x) \eta $,
		\item $\E \Delta_{(i)}(x) \Delta_{(j)}(x)
		= \partial_{(i)} f(x) \partial_{(j)} f(x) \eta^2
		+ {\Sigma(x)}_{(i,j)}\eta^2 $,
		\item $\E \prod_{j=1}^3 | \Delta_{(i_j)}(x) | = \mathcal{O}(\eta^3)$,
	\end{enumerate}
	where $\Sigma(x) :=\E {(\nabla f_\gamma(x) - \nabla f(x))}{(\nabla f_\gamma(x) - \nabla f(x))}^T$.
\end{lemma}
\begin{proof}
	We have $\Delta(x) = - \eta \nabla f_{\gamma_0}(x)$. Taking expectations, the results then follow.
\end{proof}

We now prove the main approximation theorem for the simple SGD. Before presenting the statement and proof, we shall note a few technical issues that prevents the direct application of Thm.~\ref{thm:one_step_to_multi_step} with the moment estimates in Lem.\ref{lem:Dtil} and~\ref{lem:D}. The latter suggest ignoring $\epsilon$ and setting
\begin{align*}
	b_0(x,\epsilon) = -\nabla f(x),
	\quad
	b_1(x,\epsilon) = - - \tfrac{1}{4} \nabla |\nabla f(x)|^2,
	\quad
	\sigma_0(x,\epsilon) = {\Sigma(x)}^{\tfrac{1}{2}}.
\end{align*}
Then, we would see from Lem.\ref{lem:Dtil} and~\ref{lem:D} that the SGD and the SDE have matching moments up to $\mathcal{O}(\eta^3)$. The first issue with this approach is that even if $\Sigma(x)$ is sufficiently smooth (which may follow from the regularity of $\nabla f_\gamma$), the smoothness of ${\Sigma(x)}^{\nicefrac{1}{2}}$ cannot be guaranteed unless $\Sigma(x)$ is positive-definite, which is often too strong an assumption in practice and excludes interesting cases where $\Sigma(x)$ is a singular diffusion matrix. However, the results in Sec.~\ref{sec:one_to_n_step} require smoothness. Second, we would like to consider functions $f_\gamma$ that may not have higher strong derivatives required by the Lemmas, beyond those required to define the modified equation itself. To fix both of these issues, we will use a simple mollifying technique. This is the reason for the inclusion of the $\epsilon$ parameter in the results in Sec.~\ref{sec:one_to_n_step}.

\begin{definition}
\label{def:mollification}
	Let us denote by $\nu: \R^d \rightarrow \R$, $\nu\in\Ccal^\infty_c(\R^d)$ the standard mollifier
	\begin{align*}
		\nu(x) := \begin{cases}
			C \exp( - \tfrac{1}{1 - |x|^2}) & |x| < 1 \\
			0 & |x| \geq 1,
		\end{cases}
	\end{align*}
	where $C := {(\int_{\R^d} \nu(y) dy)}^{-1}$ is chosen so that the integral of $\nu$ is 1. Further, define $\nu^\epsilon(x) = \epsilon^{-d} \nu( x / \epsilon)$. Let $\psi \in \Lcalloc^1(\R^d)$ be locally integrable, then we may define its mollification by
	\begin{align*}
		\psi^{\epsilon}(x) := (\nu^\epsilon * \psi) (x)
		= \int_{\R^d} \nu^\epsilon(x-y) \psi(y) dy
		= \int_{\Bcal(0,\epsilon)} \nu^\epsilon(y) \psi(x-y) dy,
	\end{align*}
	where $\Bcal(z,\epsilon)$ is the $d$-dimensional ball of radius $\epsilon$ centered at z. The mollification of vector (or matrix) valued functions are defined element-wise.
\end{definition}

The mollifier has very useful properties. In particular, we will use the following well-known facts (see e.g.~\cite{evans2010partial} for proof)
\begin{enumerate}[(i)]
	\item If $\psi\in\Lcalloc^1(\R^d)$, then $\psi^\epsilon \in \Ccal^\infty(\R^d)$
	\item $\psi^\epsilon(x) \rightarrow \psi(x)$ as $\epsilon\rightarrow 0$ for almost every $x\in\R^d$ (with respect to the Lebesgue measure)
	\item If $\psi$ is continuous, then $\psi^\epsilon(x) \rightarrow \psi(x)$ as $\epsilon\rightarrow 0$ uniformly on compact subsets of $\R^d$
\end{enumerate}

Next, we make use of the idea of weak derivatives.
\begin{definition}
\label{def:weak_deriv}
	Let $\Psi \in \Lcalloc^1(\R^d)$ and $J$ be a multi-index of order $|J|$. Suppose that there exists a $\psi \in \Lcalloc^1(\R^d)$ such that
	\begin{align*}
		\int_{\R^d} \Psi(x) \nabla^J \phi(x) dx
		= {(-1)}^{|J|} \int_{\R^d} \psi(x) \phi(x) dx
	\end{align*}
	for all $\phi\in \Ccal^\infty_c$. Then, we call $\psi$ the order $J$ weak derivative of $\Psi$ and write $D^J \Psi = \psi$. Note that when it exists, the weak derivative is unique almost everywhere and if $\Psi$ is differentiable, $\nabla^J \Psi = D^J \Psi$ almost everywhere \citep{evans2010partial}.
\end{definition}

The introduction of weak derivatives motivates the definition of the weak version of the function spaces $G^\alpha$.

\begin{definition}
\label{def:weak_G}
	For $\alpha \geq 1$, we define the space $G_w^\alpha$ to be the subspace of $\Lcalloc^1(\R^d)$ such that if $g \in G_w^\alpha$, then $g$ has weak derivatives up to order $\alpha$ and for each multi-index $J$ with $|J| \leq \alpha$, there exists positive integers $\kappa_1,\kappa_2$ such that
	\begin{align*}
		| D^J g (x) | \leq \kappa_1 (1 + |x|^{2\kappa_2}) \text{ for a.e. } x\in \R^d.
	\end{align*}
	As in Def.~\ref{def:poly_growth}, if $g$ depends on additional parameters, we say that $g \in G^{\alpha}_w$ if the above constants do not depend on the additional parameters. Also, vector-valued $g$ are defined as above element-wise in the co-domain. Note that $G_w^\alpha$ is a subspace of the Sobolev space $W^{\alpha,1}_{\text{loc}}$.
\end{definition}

\begin{theorem}
\label{thm:sme_sgd}
	Let, $T>0$, $\eta\in(0,1 \wedge T)$ and set $N=\lfloor T/\eta\rfloor$. Let $\{ x_k : k \geq 0 \}$ be the SGD iterations defined in~\eqref{eq:sga_iter}.
	Suppose the following conditions are met:
	\begin{enumerate}[(i)]
		\item $f \equiv \E f_\gamma$ is twice continuously differentiable, $\nabla |\nabla f|^2$ is Lipschitz, and $f\in G^4_w$.
		\item $\nabla f_\gamma$ satisfies a Lipschitz condition:
		\begin{align*}
			| \nabla f_\gamma (x) - \nabla f_\gamma (y) |
			\leq L_\gamma | x - y |
			\quad
			a.s.
		\end{align*}
		for all $x,y \in \R^d$, where $L_\gamma$ is a random variable which is positive a.s.\,and $\E L^m_\gamma < \infty$ for each $m\geq 1$.
	\end{enumerate}
	Define $\{ X_t: t\in[0,T] \}$ as the stochastic process satisfying the SDE
	\begin{align}
	\label{eq:sgd_sme_order2}
		dX_{t} = -\nabla (f(X_t) + \tfrac{1}{4}\eta
		| \nabla f(X_t) |^2)dt + \sqrt{\eta} {\Sigma(X_t)}^{\nicefrac{1}{2}}dW_t
		\qquad
		X_0 = x_0,
	\end{align}
	with $\Sigma(x) = \E {(\nabla f_\gamma(x) - \nabla f(x))}{(\nabla f_\gamma(x) - \nabla f(x))}^T$.
	Then, $\{ X_t : t\in[0,T]\}$ is an order-2 weak approximation of the SGD, i.e.\,for each $g\in G^{3}$, there exists a constant $C>0$ independent of $\eta$ such that
	\begin{align*}
		\max_{k=0,\dots,N} | \E g(x_k) - \E g(X_{k\eta}) | \leq C \eta^2.
	\end{align*}
\end{theorem}

\begin{proof}
	First, we check that Eq.~\eqref{eq:sgd_sme_order2} admits a unique solution, which amounts to checking the conditions in Thm.~\ref{app:thm:exist_uniq}. Note that the Lipschitz condition (ii) implies $\nabla f$ is Lipschitz with constant $\E L_\gamma$. To see that $\Sigma(x)^{\nicefrac{1}{2}}$ is also Lipschitz, observe that $u(x):=\nabla f_\gamma(x) - \nabla f(x)$ is Lipschitz (in the sense of (ii), with constant at most $L_\gamma + \E L_\gamma$), and
	\begin{align*}
		| {\Sigma(x)}^{\nicefrac{1}{2}} - {\Sigma(y)}^{\nicefrac{1}{2}} |
		=&
		\left|
			\| {[u(x){u(x)}^T]}^{\nicefrac{1}{2}} \|_{\Lcal^2(\Omega)}
			-
			\| {[u(y){u(y)}^T]}^{\nicefrac{1}{2}} \|_{\Lcal^2(\Omega)}
		\right| \\
		\leq&
		\| {[u(x){u(x)}^T]}^{\nicefrac{1}{2}}
		-
		{[u(y){u(y)}^T]}^{\nicefrac{1}{2}} \|_{\Lcal^2(\Omega)}.
	\end{align*}
	Moreover, observe that for vectors $u\in\R^d$ the mapping $u \mapsto {(uu^T)}^{\nicefrac{1}{2}} = uu^T / |u|$ is Lipschitz, which implies
	\begin{align*}
		| {\Sigma(x)}^{\nicefrac{1}{2}} - {\Sigma(y)}^{\nicefrac{1}{2}} |
		\leq L' \| u(x) - u(y) \|_{\Lcal^2(\Omega)}
		\leq L'' | x - y |.
	\end{align*}
	The Lipschitz conditions on the drift and the diffusion matrix imply uniform linear growth, so by Thm.~\ref{app:thm:exist_uniq}, Eq.~\eqref{eq:sgd_sme_order2} admits a unique solution.

	For each $\epsilon \in (0,1)$, define the mollified functions
	\begin{align*}
		b_0(x,\epsilon) = - \nu^\epsilon * \nabla f(x),
		\quad
		b_1(x,\epsilon) = - \tfrac{1}{4}\nu^\epsilon * (\nabla |\nabla f(x)|^2),
		\quad
		\sigma_0(x,\epsilon) = \nu^\epsilon * {\Sigma(x)}^{\nicefrac{1}{2}}.
	\end{align*}
	Observe that $b_0 + \eta b_1,\sigma_0$ satisfies a Lipschitz condition in $x$ uniformly in $\eta,\epsilon$. To see this, note that for any Lipschitz function $\psi$ with constant $L$, we have
	\begin{align*}
		| \nu^\epsilon * \psi(x) - \nu^\epsilon * \psi(y) |
		\leq& \int_{\Bcal(0,\epsilon)} \nu^\epsilon(z)
		| \psi(x - z) - \psi(y - z) | dz
		\leq L |x - y|,
	\end{align*}
	which proves $b_0 + \eta b_1$ and $\sigma_0$ are uniformly Lipschitz. Similarly, the linear growth condition follows. Hence, we may define a family of stochastic processes $\{X^{\epsilon}_t: \epsilon\in(0,1)\}$ satisfying
	\begin{align*}
		dX^{\epsilon}_{t} = b_0(X^{\epsilon}_{t}, \epsilon)
		+ \eta b_1(X^{\epsilon}_{t}, \epsilon)
		+ \sqrt{\eta} {\sigma_0(X^{\epsilon}_{t}, \epsilon)} dW_t
		\qquad
		X^{\epsilon}_0 = x_0,
	\end{align*}
	% \begin{align*}
	% 	dX^{\epsilon}_{t} = -\nabla (f^{\epsilon}(X^{\epsilon}_t) + \tfrac{1}{4}\eta
	% 	(\nu^\epsilon * | \nabla f|^2) (X_t) )dt + \sqrt{\eta}
	% 	{[\epsilon I + \Sigma^{\epsilon}(X_t)]}^{\nicefrac{1}{2}}dW_t
	% 	\qquad
	% 	X^{\epsilon}_0 = x_0,
	% \end{align*}
	which each admits a unique solution by Thm.~\ref{app:thm:exist_uniq}. Now, we claim that $b_0(\cdot,\epsilon),b_1(\cdot,\epsilon),\sigma_0(\cdot,\epsilon) \in G^3$ uniformly in $\epsilon$. To see this, simply observe that mollifications are smooth, and moreover, the polynomial growth is satisfied since $ \nu^\epsilon * D^J \psi = \nabla^J (\nu^\epsilon * \psi)$ and furthermore, if $\psi\in G$, then we have
	\begin{align*}
		| \psi^\epsilon(x) |
		\leq& \int_{\Bcal(0,\epsilon)}
		\nu^\epsilon(y) | \psi(x-y) | dy \\
		\leq& \kappa_1
		\left(
			1 + 2^{2\kappa_2-1} | x |^{2\kappa_2} +
			2^{2\kappa_2-1}
			\tfrac{1}{\epsilon^{d}}
			\int_{\Bcal(0,\epsilon)}
			| y |^{2\kappa_2} dy
		\right)
	\end{align*}
	But $\int_{\Bcal(0,\epsilon)} | y |^{2\kappa_2} dy \leq \text{Vol}(\Bcal(0,\epsilon)) = C \epsilon^{d}$, where $C$ is independent of $\epsilon$. This shows that $\psi^\epsilon \in G$ uniformly in $\epsilon$. This immediately implies that $b_0(\cdot,\epsilon),b_1(\cdot,\epsilon),\sigma_0(\cdot,\epsilon) \in G^3$.

	Now, since $b_0(x,\epsilon) \rightarrow b_0(x,0)$ (and similarly for $b_1, \sigma_0$), and the limits are continuous, by Lem.~\ref{lem:Dtil},~\ref{lem:D},~\ref{app:lem:sgd_moment_estimate},~\ref{app:lem:weak_deriv_poly_growth},, all conditions of Thm.~\ref{thm:one_step_to_multi_step} are satisfied, and hence we conclude that for each $g\in G^3$, we have,
	\begin{align*}
		\max_{k=0,\dots,N}
		| \E g(X^\epsilon_{k\eta}) - \E g(x_k) |
		\leq C (\eta^2 + \rho(\epsilon)),
	\end{align*}
	where $C$ is independent of $\eta$ and $\epsilon$ and $\rho(\epsilon) \rightarrow 0$ as $\epsilon \rightarrow 0$. Moreover, since $b_0(x,\epsilon) \rightarrow b_0(x,0)$ (and similarly for $b_1, \sigma_0$) uniformly on compact sets, we may apply Thm.~\ref{app:thm:limit_q} to conclude that
	\begin{align*}
		\sup_{t\in[0,T]} \E | X^{\epsilon}_t - X_t |^2 \rightarrow 0
		\text{ as }
		\epsilon \rightarrow 0.
	\end{align*}
	Thus, we have
	\begin{align*}
		&| \E g(X_{k\eta}) - \E g(x_k) | \\
		\leq&
		| \E g(X^\epsilon_{k\eta}) - \E g(x_k) |
		+ | \E g(X^\epsilon_{k\eta}) - \E g(X_{k\eta}) | \\
		\leq& C(\eta^2 + \rho(\epsilon))
		+ {\left(
			\E | X^\epsilon_{k\eta} - X_{k\eta}  |^2
		\right)}^{\nicefrac{1}{2}} \\
		& \times
		{\left(
			\int_{0}^{1} \E
			|\nabla^2 g
			(\lambda X^\epsilon_{k\eta} + (1-\lambda) X_{k\eta}) |^2
			d\lambda
		\right)}^{\nicefrac{1}{2}}
	\end{align*}
	Using Thm.~\ref{app:thm:moment_estimates} and assumption that $\nabla^2 g\in G$, the last expectation is finite and hence taking the limit $\epsilon\rightarrow 0$ yields our result.
\end{proof}

By going for a lower order approximation, we of course have the following:
\begin{corollary}
\label{cor:sme_sgd_1st_order}
	Assume the same conditions as in Thm.~\ref{thm:sme_sgd}, except that we replace (i) with
	\begin{enumerate}[(i)']
		\item $f \equiv \E f_\gamma$ is continuously differentiable, and $f\in G^3_w$.
	\end{enumerate}
	Define $\{ X_t: t\in[0,T] \}$ as the stochastic process satisfying the SDE
	\begin{align}
	\label{eq:sgd_sme_order1}
		dX_{t} = -\nabla f(X_t) dt + \sqrt{\eta} {\Sigma(X_t)}^{\nicefrac{1}{2}}dW_t
		\qquad
		X_0 = x_0,
	\end{align}
	with $\Sigma(x) = \E {(\nabla f_\gamma(x) - \nabla f(x))}{(\nabla f_\gamma(x) - \nabla f(x))}^T$.
	Then, $\{ X_t : t\in[0,T]\}$ is an order-1 weak approximation of the SGD, i.e.\,for each $g\in G^{2}$, there exists a constant $C>0$ independent of $\eta$ such that
	\begin{align*}
		\max_{k=0,\dots,N} | \E g(X_{k\eta}) - \E g(x_k) | \leq C \eta.
	\end{align*}
\end{corollary}

\begin{remark}
	In the above results, the most restrictive condition is probably the Lipschitz condition on $\nabla f_\gamma$. Such Lipschitz conditions are important to ensure that the SMEs admit unique strong solutions and the SGA having uniformly bounded moments. Note that following similar techniques in SDE analysis (e.g.~\cite{Kloeden1992}), these global conditions may be relaxed to their respective local versions if we assume in addition a uniform global linear growth condition on $\nabla f_\gamma$. Finally, for applications, typical loss functions have inward pointing gradients for all sufficiently large $x$, meaning that the SGD iterates will be uniformly bounded almost surely. Thus, we may simply modify the loss functions for large $x$ (without affecting the SGA iterates) to satisfy the conditions above.
\end{remark}

\begin{remark}
	The constant $C$ does not depend on $\eta$, but as evidenced in the proof of the theorem, it generally depends on $g$, $T$, $d$ and the various Lipschitz constants. For the fairly general situation we are consider, we do not derive tight estimates of these dependencies.
\end{remark}

\subsection{SME for stochastic gradient descent with momentum}
\label{sec:sme_msgd}

Let us discuss the corresponding SME for a popular variant of the SGD called the \emph{momentum SGD} (MSGD). The momentum SGD augments the usual SGD iterations with a ``memory'' term. In the usual form, we have the iterations
\begin{align*}
	\hat{v}_{k+1} &= \hat{\mu} \hat{v}_k - \hat{\eta} \nabla f_{\gamma_k}(x_k) \\
	x_{k+1} &= x_{k} + \hat{v}_{k+1}
\end{align*}
where $\hat{\mu}\in (0,1)$ (typically close to 1) is called the \emph{momentum parameter} and $\hat{\eta}$ is the learning rate. Let us consider a rescaled version of the above that is easier to analyze via continuous-time approximations. We redefine
\begin{align}\label{eq:rescale}
	\eta := \sqrt{\hat{\eta}},
	\qquad
	v_k := \hat{v}_k/\sqrt{\hat{\eta}},
	\qquad
	\mu := (1-\hat{\mu}) / \sqrt{\hat{\eta}}
\end{align}
 to obtain
\begin{align}
\label{eq:msgd_iter}
	\begin{split}
		v_{k+1} &= v_k - \mu \eta v_k - \eta \nabla f_{\gamma_k}(x_k) \\
		x_{k+1} &= x_{k} + \eta v_{k+1}.
	\end{split}
\end{align}
In view of the rescaling, the range of momentum parameters we consider becomes $\mu\in (0,\eta^{-1/2})$, which we may replace by $(0,\infty)$ for simplicity.

Let us now derive the SME satisfied by the iterations~\eqref{eq:msgd_iter}. Observe that this is again a special case of~\eqref{eq:milstein_discrete} with $x$ now replaced by $(v,x)$ and
\begin{align*}
	h(v, x, \gamma, \eta) = (
		- \mu v - \nabla f_\gamma (x),
		 v - \eta \mu v - \eta \nabla f_\gamma (x))
\end{align*}
In view of Thm.~\ref{thm:sme_msgd} and the results in Sec.~\ref{sec:sme_sgd}, in order to derive the SMEs we simply match moments up to order 3. As in Sec.~\ref{sec:sme_sgd}, let us define the one step difference
\begin{align}
\label{eq:Delta_defn_mom}
	\Delta(v,x) := (v^{v,x,0}_1 - v, x^{v,x,0}_1 - x).
\end{align}
The following moment expansions are immediate.
\begin{lemma}
\label{lem:D_mom}
	Let $\Delta(x,v)$ be defined as in~\eqref{eq:Delta_defn_mom}. We have
	\begin{enumerate}[(i)]
		\item $\E \Delta_{(i)}(v, x) =
		\eta (
			-\mu v_{(i)} - \partial_{(i)} f(x),
			v ) + \eta^2 (0, - \mu v_{(i)} - \partial_{(i)} f(x))$,
		\item $\E \Delta_{(i)}(v, x) \Delta_{(j)}(v, x)
		= \\ \eta^2 \begin{pmatrix}
			\mu^2 v_{(i)}v_{(j)} + \mu v_{(i)} \partial_{(j)} f(x)
			+ \mu v_{(j)} \partial_{(i)} f(x) \\
			+ \Sigma(x)_{(i,j)} + \partial_{(i)} \partial_{(j)} f(x)
			& -\mu v_{(i)}v_{(j)} - v_{(i)} \partial_{(j)} f(x) \\ \\
			-\mu v_{(i)}v_{(j)} - v_{(j)} \partial_{(i)} f(x)
			& v_{(i)} v_{(j)} \\
		\end{pmatrix}
		\\ + \mathcal{O}(\eta^3)$,
		\item $\E \prod_{j=1}^3 | \Delta_{(i_j)}(v,x) | = \mathcal{O}(\eta^3)$,
	\end{enumerate}
	where $\Sigma(x) :=\E {(\nabla f_\gamma(x) - \nabla f(x))}{(\nabla f_\gamma(x) - \nabla f(x))}^T$.
\end{lemma}
\begin{proof}
	The proof follows from direct calculation of the moments.
\end{proof}

Hence, proceeding exactly as in Sec.~\ref{sec:sme_sgd} and using Lem.\ref{lem:Dtil},~\ref{lem:D_mom}, we see that we may set
\begin{align*}
	b_0(v, x) &= (-\mu v - \nabla f(x), v) \\
	b_1(v, x) &= -\tfrac{1}{2} \left(
		\mu[ \mu v + \nabla f(x) ] - \nabla^2 f(x) v,
		\mu v + \nabla f(x)
	\right) \\
	% b_1(v, x) &=  \left(
	% 	-\tfrac{1}{2} [ (1+\mu^2) v + \mu \nabla f ],
	% 	- (I+ \tfrac{1}{2} \nabla^2 f) (\mu v + \nabla f)
	% \right) \\
	\sigma_0(v, x) &=
		\begin{pmatrix}
			\Sigma(x)^{\nicefrac{1}{2}} & 0 \\
			0 & 0 \\
		\end{pmatrix}
\end{align*}
in order to match the moments. By similar mollification and limiting arguments as in Thm.~\ref{thm:sme_sgd}, we arrive at the following approximation theorem, where we can see that the SME for MSGD takes the form of a Langevin equation.
\begin{theorem}
\label{thm:sme_msgd}
	Assume the same conditions as in Thm.~\ref{thm:sme_sgd}. Let $\mu > 0$ be fixed and
	define $\{ V_t, X_t: t\in[0,T] \}$ as the stochastic process satisfying the SDE
	\begin{align}
		\label{eq:msgd_sme_order2}
			&dV_{t} = -
			[
				(\mu I + \tfrac{1}{2} \eta [\mu^2 I - \nabla^2 f(X_t)]) V_t
				+ (1 + \tfrac{1}{2}\eta\mu) \nabla f(X_t)
			] dt
			+ \sqrt{\eta} {\Sigma(X_t)}^{\nicefrac{1}{2}}dW_t
			\quad V_0 = v_0, \nonumber \\
			&dX_{t} =
			[
				(1 - \tfrac{1}{2}\eta\mu) V_t
				- \tfrac{1}{2} \eta \nabla f(X_t)
			] dt
			\quad X_0 = x_0,
	\end{align}
	with $\Sigma(x)$ as defined in Thm.~\ref{thm:sme_sgd}.
	Then, $\{ (V_t, X_t) : t\in[0,T]\}$ is an order-2 weak approximation of the MSGD.

	Moreover, if we relax the assumptions to Cor.~\ref{cor:sme_sgd_1st_order}, we have the order-1 weak approximation
	\begin{align}
		\label{eq:msgd_sme_order1}
			&dV_{t} = -
			[
				\mu V_t + \nabla f(X_t)
			] dt
			+ \sqrt{\eta} {\Sigma(X_t)}^{\nicefrac{1}{2}}dW_t
			\quad V_0 = v_0, \nonumber \\
			&dX_{t} =
			V_t dt
			\quad X_0 = x_0.
	\end{align}
\end{theorem}
Note that by inverting the scaling~\eqref{eq:rescale}, the order-1 SME~\eqref{eq:msgd_sme_order1} is the formal equation derived in~\cite{li2015dynamics}.

\subsection{SME for a momentum variant: Nesterov accelerated gradient}
\label{sec:sme_snag}

It follows from the calculation above that we can also obtain the SME for the stochastic gradient version of the Nesterov accelerated gradient (NAG) method~\citep{nesterov1983method}, which we refer to as SNAG.
In the non-stochastic case, the NAG method has been analyzed using the ODE approach~\citep{su2014differential}. Therefore, the derivations in this section can be viewed as a stochastic parallel. The NAG method is sometimes used with stochastic gradients, and hence it is useful to analyze its properties in this setting and compare it to MSGD.

The unscaled NAG iterations are
\begin{align*}
	\hat{v}_{k+1} &= \hat{\mu}_k \hat{v}_k - \hat{\eta} \nabla f_{\gamma_k}(x_k + \hat{\mu}_k \hat{v}_k) \\
	x_{k+1} &= x_{k} + \hat{v}_{k+1}
\end{align*}
with $\hat{v}_0=0$, which differs from the momentum iterations as the gradient is now evaluated at a ``predicted'' position $x_k+\hat{\mu}_k \hat{v}_k$, instead of the original position $x_k$. Moreover, the momentum parameter $\hat{\mu}_k$ is now allowed to vary as $k$ increases, and in fact, the usual choice of
\begin{align}\label{eq:nesterov_mu}
	\hat{\mu}_k = \tfrac{k-1}{k+2}
\end{align}
this has important links to stability and acceleration in the deterministic case~\citep{nesterov1983method,su2014differential}. In particular, it achieves $\mathcal{O}(1/k^2)$ convergence rate for general convex functions. On the other hand, a constant $\hat{\mu}_k$ is suggested for strongly convex functions~\citep{nesterov2013introductory}. In the following, we shall first consider the case of constant momentum parameter with $\hat{\mu}_k \equiv \hat{\mu}$, and then the choice~\eqref{eq:nesterov_mu} subsequently.

\paragraph{Constant momentum.}
Using the same rescaling in~\eqref{eq:rescale}, we have
\begin{align}
\label{eq:nag_iter}
	\begin{split}
		v_{k+1} &= v_k - \mu \eta v_k - \eta \nabla f_{\gamma_k}(x_k + \eta (1 - \mu \eta) v_k) \\
		x_{k+1} &= x_{k} + \eta v_{k+1}.
	\end{split}
\end{align}
which is again~\eqref{eq:milstein_discrete} with
\begin{align*}
	h(v, x, \gamma, \eta) = (
		- \mu v - \nabla f_\gamma (x + \eta (1-\mu\eta) v),
		v - \eta \mu v - \eta \nabla f_\gamma (x + \eta (1-\mu\eta)v))
\end{align*}
Hence, we have the following moment expansion.
\begin{lemma}
\label{lem:D_nag}
	Let $\Delta(x,v) := (v^{v,x,0}_1 - v, x^{v,x,0}_1 - x)$. We have
	\begin{enumerate}[(i)]
		\item $\E \Delta_{(i)}(v, x) =
		\eta (
			-\mu v_{(i)} - \partial_{(i)} f(x), v ) \\
		+ \eta^2 (
			\partial_{(i)}\partial_{(j)} f(x) v_{(j)}, - \mu v_{(i)} - \partial_{(i)} f(x+v))
			+ \mathcal{O}(\eta^3)$,
		\item $\E \Delta_{(i)}(v, x) \Delta_{(j)}(v, x)
		= \\ \eta^2 \begin{pmatrix}
			\mu^2 v_{(i)}v_{(j)} + \mu v_{(i)} \partial_{(j)} f(x+v)
			+ \mu v_{(j)} \partial_{(i)} f(x+v) \\
			+ \Sigma(x+v)_{(i,j)} + \partial_{(i)} \partial_{(j)} f(x+v)
			& -\mu v_{(i)}v_{(j)} - v_{(i)} \partial_{(j)} f(x+v) \\ \\
			-\mu v_{(i)}v_{(j)} - v_{(j)} \partial_{(i)} f(x+v)
			& v_{(i)} v_{(j)} \\
		\end{pmatrix}
		\\ + \mathcal{O}(\eta^3)$,
		\item $\E \prod_{j=1}^3 | \Delta_{(i_j)}(v,x) | = \mathcal{O}(\eta^3)$,
	\end{enumerate}
	where $\Sigma(x) :=\E {(\nabla f_\gamma(x) - \nabla f(x))}{(\nabla f_\gamma(x) - \nabla f(x))}^T$.
\end{lemma}
\begin{proof}
	The proof follows from direct calculation of the moments and Taylor's expansion.
\end{proof}
Hence, we may match moments by setting
\begin{align*}
	b_0(v, x) &= (-\mu v - \nabla f(x), v) \\
	b_1(v, x) &= -\tfrac{1}{2} \left(
		\mu[ \mu v + \nabla f(x) ] + \nabla^2 f(x) v,
		\mu v + \nabla f(x)
	\right) \\
	% b_1(v, x) &=  \left(
	% 	-\tfrac{1}{2} [ (1+\mu^2) v + \mu \nabla f ],
	% 	- (I+ \tfrac{1}{2} \nabla^2 f) (\mu v + \nabla f)
	% \right) \\
	\sigma_0(v, x) &=
		\begin{pmatrix}
			\Sigma(x)^{\tfrac{1}{2}} & 0 \\
			0 & 0 \\
		\end{pmatrix}
\end{align*}
from which we obtain the following approximation theorem for SNAG.
\begin{theorem}
\label{thm:sme_nag}
	Assume the same conditions as in Thm.~\ref{thm:sme_msgd}.
	Define $\{ V_t, X_t: t\in[0,T] \}$ as the stochastic process satisfying the SDE
	\begin{align}
		\label{eq:nag_sme_order2}
			&dV_{t} = -
			[
				(\mu I + \tfrac{1}{2} \eta [\mu^2 I + \nabla^2 f(X_t)]) V_t
				+ (1 + \tfrac{1}{2}\eta\mu) \nabla f(X_t)
			] dt
			+ \sqrt{\eta} {\Sigma(X_t)}^{\nicefrac{1}{2}}dW_t
			\quad V_0 = v_0, \nonumber \\
			&dX_{t} =
			[
				(1 - \tfrac{1}{2}\eta\mu) V_t
				- \tfrac{1}{2} \eta \nabla f(X_t)
			] dt
			\quad X_0 = x_0,
	\end{align}
	with $\Sigma$ as defined in Thm.~\ref{thm:sme_msgd}. Then, $\{ (V_t, X_t) : t\in[0,T]\}$ is an order-2 weak approximation of SNAG. Moreover, the same order-1 weak approximation of MSGD in~\eqref{eq:msgd_sme_order1} holds for the SNAG.
\end{theorem}
The result above shows that for constant momentum parameters, the modified equations for MSGD and the SNAG are equivalent at leading order, but differ when we consider the second order modified equation. Let us now discuss the case where the momentum parameter is allowed to vary.

\paragraph{Varying momentum.}
Now let us take $\hat{\mu}$ as in~\eqref{eq:nesterov_mu}. Then, using the same rescaling arguments, we arrive at
\begin{align}
\label{eq:nag_dyn_iter}
	\begin{split}
		v_{k+1} &= v_k - \mu_k \eta v_k - \eta \nabla f_{\gamma_k}(x_k + \eta (1 - \mu_k \eta) v_k) \\
		x_{k+1} &= x_{k} + \eta v_{k+1}.
	\end{split}
\end{align}
with $\mu_k = 3/(2\eta+k\eta)$. Now, in order to apply our theoretical results to deduce the SME, simply notice that we may introduce an auxiliary scalar variable
\begin{align*}
	z_{k+1} = z_{k} + \eta, \qquad z_{0} = 0.
\end{align*}
Then, $\mu_k = 3/(2\eta + z_k)$, and hence all terms are now not explicitly $k$-independent, thus we may proceed formally as in the previous sections to arrive at the order-1 SME for SNAG with varying momentum
\begin{align}
	\label{eq:snag_dyn_sme_order1}
		&dV_{t} = -
		[
			\tfrac{3}{t} V_t + \nabla f(X_t)
		] dt
		+ \sqrt{\eta} {\Sigma(X_t)}^{\nicefrac{1}{2}}dW_t
		\quad V_0 = 0, \nonumber \\
		&dX_{t} =
		V_t dt
		\quad X_0 = x_0.
\end{align}
This result is formal because the term $3/t$ does not satisfy our global Lipschitz conditions, unless we restrict our interval to some $[t_0,T]$ with $t_0>0$, in which case the above result becomes rigorous. Alternatively, some limiting arguments have to be used to establish well-posedness of the equation on $[0,T]$ individually. We shall omit these analyses in the current paper, and only consider~\eqref{eq:snag_dyn_sme_order1} on some interval $[t_0,T]$, where initial conditions are then replaced by $(v_{t_0}, x_{t_0})$. As a point of comparison,~\eqref{eq:snag_dyn_sme_order1} reduces to the ODE derived in~\cite{su2014differential} if $\Sigma(x) \equiv 0$ (i.e. the gradients are non-stochastic).

\section{Applications of the SMEs to the analysis of SGA}
\label{sec:applications}

In this section, we apply the SME framework developed to analyze the dynamics of the three stochastic gradient algorithm variants discussed above, namely SGD, MSGD and SNAG. We shall focus on simple but non-trivial models where to a large extent, analytical computations using SME are tractable, giving us key insights into the algorithms that are otherwise difficult to obtain without appealing to the continuous formalism presented in this paper.
We consider primarily the following model:
\paragraph{Model:}
Let $H \in \R^{d\times d}$ be a symmetric, positive definite matrix.
Define the sample objective
\begin{align}\label{eq:model_1}
	&f_\gamma(x) := \tfrac{1}{2} {(x - \gamma)}^T H {(x - \gamma)}
	- \tfrac{1}{2} \tr(H) \nonumber \\
	&\gamma \sim \mathcal{N}(0, I)
\end{align}
which gives the total objective $f(x) \equiv \E f_\gamma (x) = \tfrac{1}{2} x^T H x$.

\subsection{SME analysis of SGD}

We first derive the SME associated with~\eqref{eq:model_1}. For simplicity, we will only consider the order-1 SME~\eqref{eq:sgd_sme_order1}. A direct computation shows that $\Sigma(x) = H^2$
and so the SME for SGD applied to model~\eqref{eq:model_1} is
\begin{align*}
	dX_t = - H X_t dt + \sqrt{\eta} H dW_t,
\end{align*}
This is a multi-dimensional Ornstein-Uhlenbeck (OU) process and admits the explicit solution
\begin{align*}
	X_t = e^{-t H}
	\left(
		x_0 + \sqrt{\eta} \int_{0}^{t} e^{s H} H dW_s
	\right).
\end{align*}
Observe that for each $t\geq 0$, the distribution of $X_t$ is Gaussian. Using It\^{o}'s isometry, we then deduce the dynamics of the objective function
\begin{align}\label{eq:sgd_mod1_exactf}
	\E f(X_t)
	=& \tfrac{1}{2} x_0^T H e^{- 2 t H} x_0
	+ \tfrac{1}{2} \eta \int_{0}^{t} \tr ( H^3 e^{-2(t-s) H} ) ds \nonumber \\
	=& \tfrac{1}{2} x_0^T H e^{- 2 t H} x_0
	+  \tfrac{1}{4} \eta \sum_{i=1}^{n} \lambda^2_{i}(H) (1 - e^{-2t\lambda_i(H)}).
\end{align}
The first term decays linearly with asymptotic rate $2\lambda_d(H)$, and the second term is induced by noise, and its asymptotic value is proportional to the learning rate $\eta$. This is the well-known two-phase behavior of SGD under constant learning rates: an initial descent phase induced by the deterministic gradient flow and an eventual fluctuation phase dominated by the variance of the stochastic gradients. In this sense, the SME makes the same predictions, and in fact we can see that it approximates the SGD iterations well as $\eta$ decreases (Fig.~\ref{fig:sgd_conv_cond}(a)), according to the rates we derived in Thm.~\ref{thm:sme_sgd} and Cor.~\ref{cor:sme_sgd_1st_order}.

\begin{figure}[h]
	\begin{center}
		% \subfloat[]{\includegraphics[width=7cm]{sgd_mod1_conv.pdf}}
		\subfloat[]{\includegraphics[width=7cm]{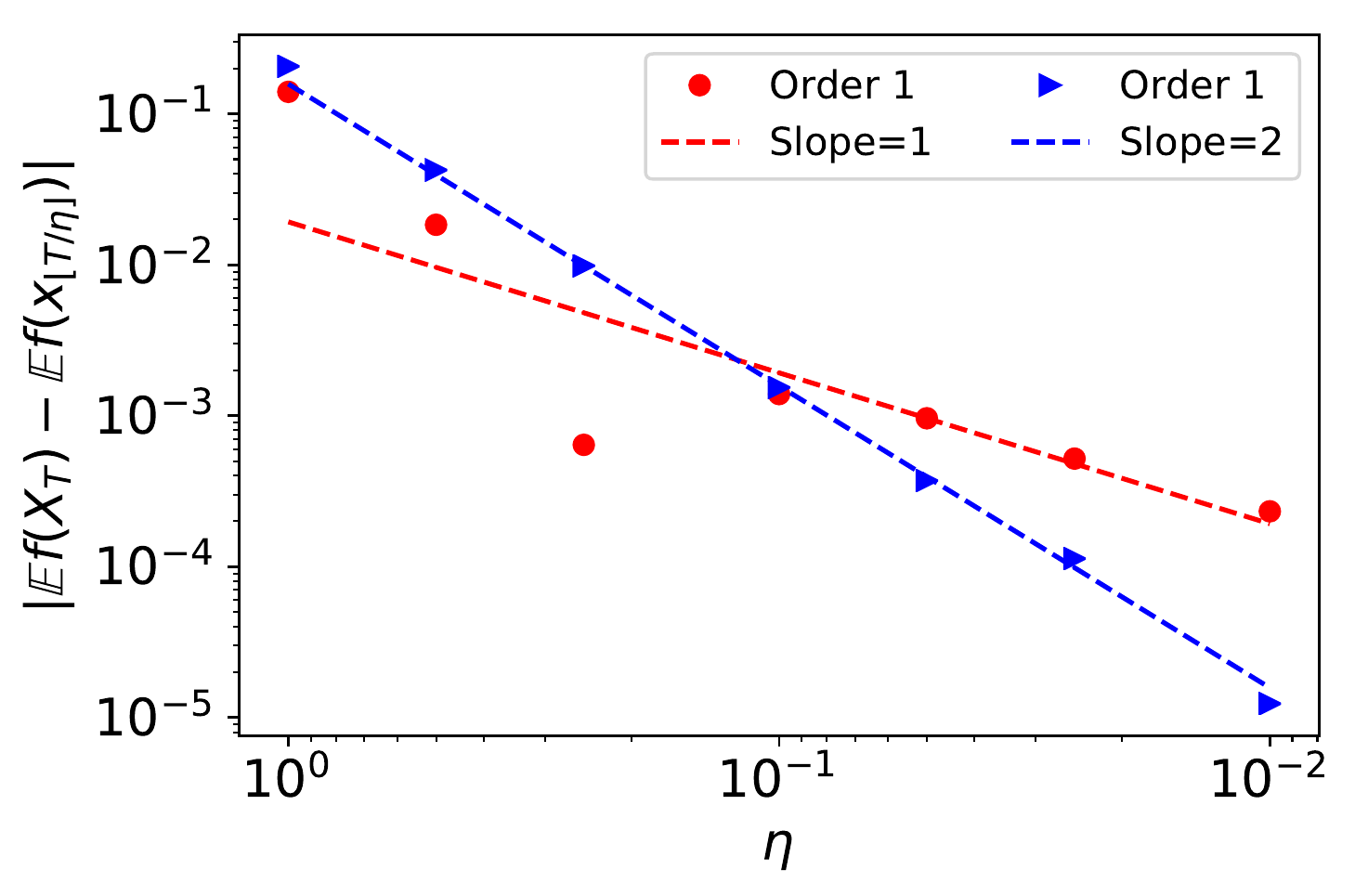}}
		\subfloat[]{\includegraphics[width=7cm]{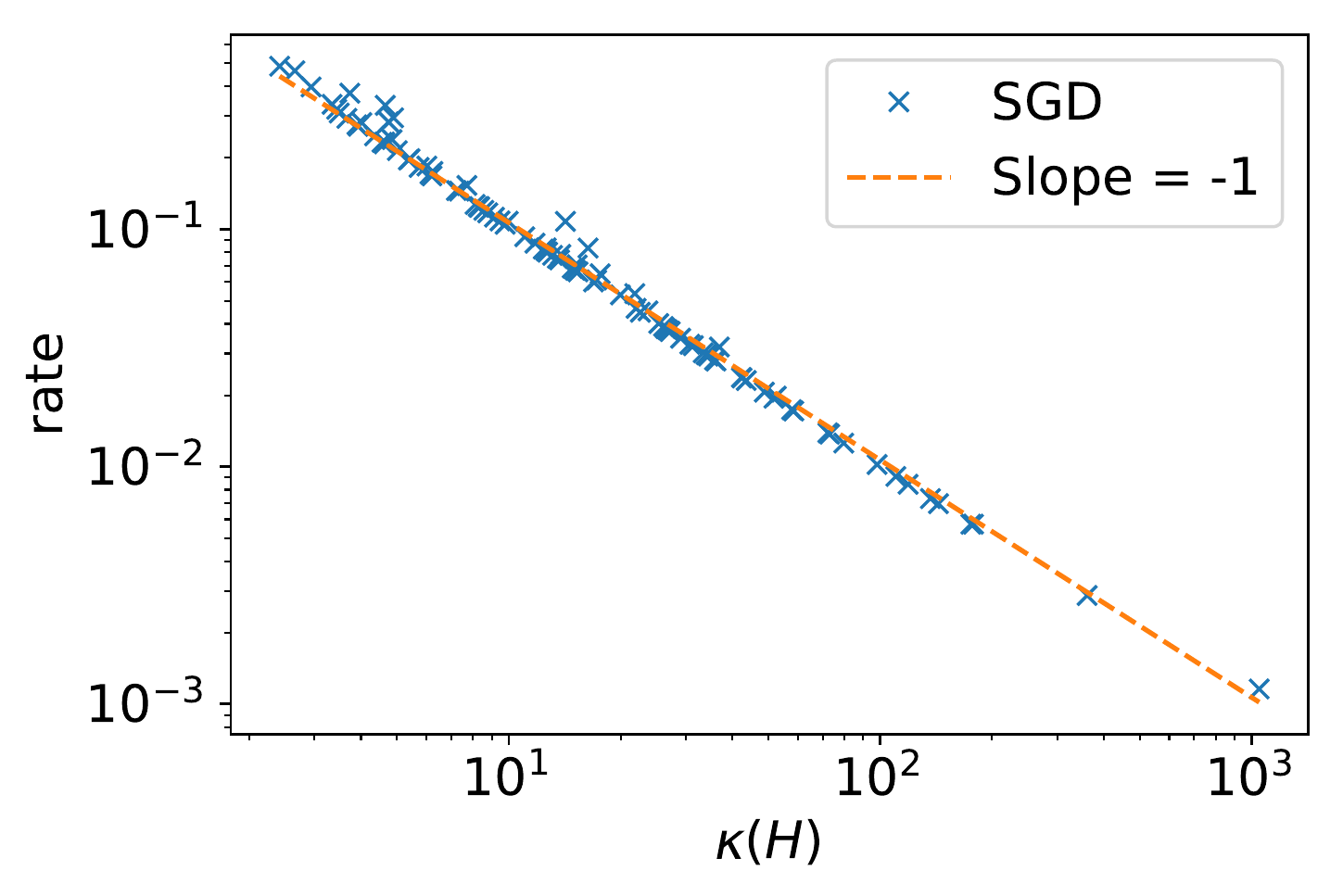}}
	\end{center}
	\caption{SME prediction vs SGD dynamics. (a) SME as a weak approximation of the SGD. We compute the weak error with test function $g$ equal to $f$ (see Thm.~\ref{thm:sme_sgd}). As predicted by our analysis, the order-2 SME~\eqref{eq:sgd_sme_order2} (order-1 SME~\eqref{eq:sgd_sme_order1}) should give a slope = 2 (1) decrease in error as $\eta$ decreases (note that the x-axis is flipped). The SME solution is computed using an exact formula derived by the application of It\^{o} isometry and the SGD expectation is averaged over 1e6 runs. We took $T=2.0$. We see that the predictions of Thm.~\ref{thm:sme_sgd} and~Cor.\ref{cor:sme_sgd_1st_order} hold. (b) Descent rate vs condition number. $H$ is generated with different condition numbers, and the resulting descent rate of SGD is approximately $\propto {\kappa(H)}^{-1}$, as predicted by the SME.}
	\label{fig:sgd_conv_cond}
\end{figure}
Moreover, notice that by the identification $t=k\eta$ ($k$ is the SGD iteration number), the SME analysis tells us that the asymptotic linear convergence rate (in $k$, i.e.\,$\text{rate}\sim -\log[\E f(x_k)] / k$) in the descent phase of the SGD is $2 \lambda_d(H)\eta$. For numerical stability (even in the non-stochastic case), we usually require $\eta \propto 1/\lambda_1(H)$, thus the maximal descent rate is inversely proportional to the condition number $\kappa(H)=\lambda_1(H)/\lambda_d(H)$. We validate this observation by generating a collection of $H$'s with varying condition numbers and applying SGD with $\eta \propto 1/ \lambda_1(H)$. In Fig~\ref{fig:sgd_conv_cond}(b), we plot the initial descent rates versus the condition number of $H$ and we observe that we indeed have $\text{rate}\propto {\kappa(H)}^{-1}$.

\paragraph{Alternate model.}
Now, we consider a slight variation of the model~\eqref{eq:model_1}. The goal is show that the dynamics of SGD (and the corresponding SME) is not always Gaussian-like and thus using the OU process to model the SGD is not always valid.
Given the same positive-definite matrix $H$, we diagonalize it in the form $H = Q D Q^T$ where $Q$ is an orthogonal matrix and $D$ is a diagonal matrix of eigenvalues. We then define the sample objective
\begin{align}\label{eq:model_2}
	&f_\gamma(x) := \tfrac{1}{2} {(Q^T x)}^T [D + \diag(\gamma)] {(Q^T x)} \nonumber \\
	&\gamma {\sim} \mathcal{N}(0, I)
\end{align}
which gives the same total objective $f(x) \equiv \E f_\gamma (x) = \tfrac{1}{2} x^T H x$. However, we have a different expression for $\Sigma(x)$
\begin{align*}
	\Sigma(x) = Q {\diag(Q x)}^2 Q^T,
\end{align*}
which gives the SME
\begin{align*}
	dX_t &= - H X_t dt + \sqrt{\eta} {Q |\diag(Q^T x)| Q^T} dW_t \\
	&\overset{\text{in distribution}}{=}
	- H X_t dt + \sqrt{\eta} {Q \diag(Q^T x) Q^T} dW_t.
\end{align*}
We can rewrite the above as
\begin{align*}
	dX_t = - H X_t dt + \sqrt{\eta} \sum_{l=1}^{d} Q^{(l)} X_t dW_{(l),t},
\end{align*}
where $Q^{(l)} = Q \diag(Q_{(l,\cdot)}) Q^T$ and $Q_{(l,\cdot)}$ denotes the $l^\text{th}$ row of $Q$. By observing that every pair of $\{ H, Q^{(1)}, \dots, Q^{(d)} \}$ commute, we have the explicit solution
\begin{align*}
	X_t = e^{- \tfrac{1}{2} \eta t + \sqrt{\eta} \sum_{l=1}^{d} Q^{(l)} W_{(l),t}}
	e^{- H t} x_0.
\end{align*}
which is a multi-dimensional Black-Scholes~\citep{black1973pricing} type of stochastic process. In particular, the distribution is not Gaussian of any $t>0$. Nevertheless, we may take expectation to obtain
\begin{align*}
	\E f(X_t) = \tfrac{1}{2} e^{\eta t} x_0^T H e^{-2 H t} x_0.
\end{align*}
This immediately implies the following interesting behavior: if $\eta < 2 \lambda_d(H)$, then $2 H - \eta I$ is positive definite and so $\E f(X_t) \rightarrow 0$ exponentially at constant, non-zero $\eta$; Otherwise, depending on initial condition $x_0$, the objective may not converge to 0. In particular, if $\eta > 2 \lambda_d(H)$ (which happens quite often if the condition number of $H$ is large) and $x_0$ is in general position, then we have asymptotic exponential divergence. This is a variance-induced divergence typically observed in Black-Scholes and geometric Brownian motion type of stochastic processes. The term ``variance-induced'' is important here since the deterministic part of the evolution equation is mean-reverting and in fact is identical to the stable OU process studied earlier.
In Fig.~\ref{fig:sgd_bs_conv_cond}(a), (b), we show the correspondence of the SME findings with the actual dynamics of the SGD iterations. In particular, we see in Fig.~\ref{fig:sgd_bs_conv_cond}(c) that for small $\eta$, we have exponential convergence of the SGD at constant learning rates, whereas for $\eta > 2 \lambda_d(H)$, the SGD iterates start to oscillate wildly and its mean value is dominated by few large values and diverges approximately at the rate predicted by the SME. Note that this divergence is predicted to be at a finite $\eta$, and from the theory developed so far we cannot conclude that the SME approximation always holds accurately at this regime (but the approximation is guaranteed for $\eta$ sufficiently small). Nevertheless, we observe at least in this model that the variance-induced divergence of the SGD happens as predicted by the SME.
\begin{figure}[h]
	\begin{center}
		\subfloat[]{\includegraphics[width=7cm]{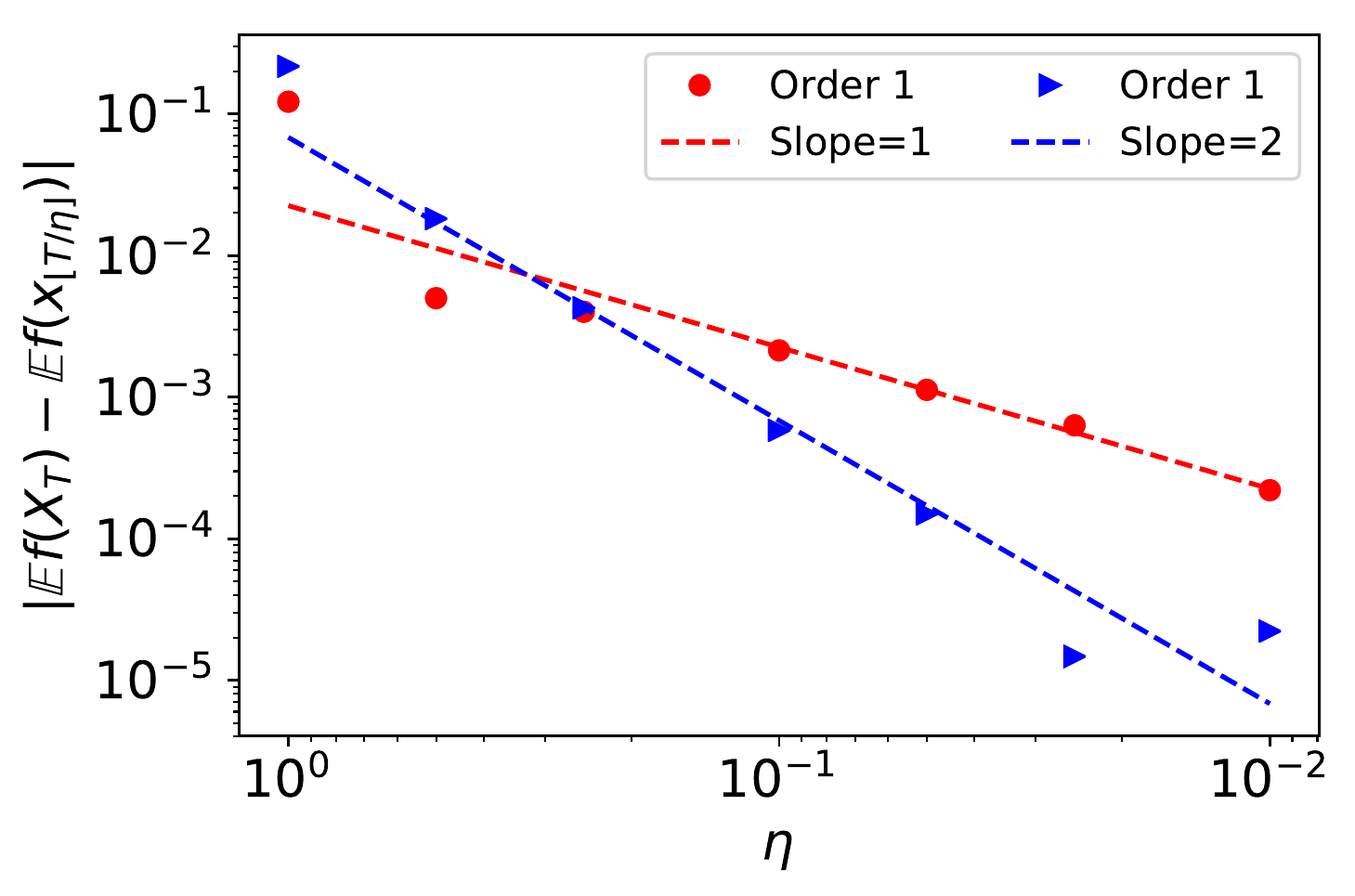}}
		\subfloat[]{\includegraphics[width=7cm]{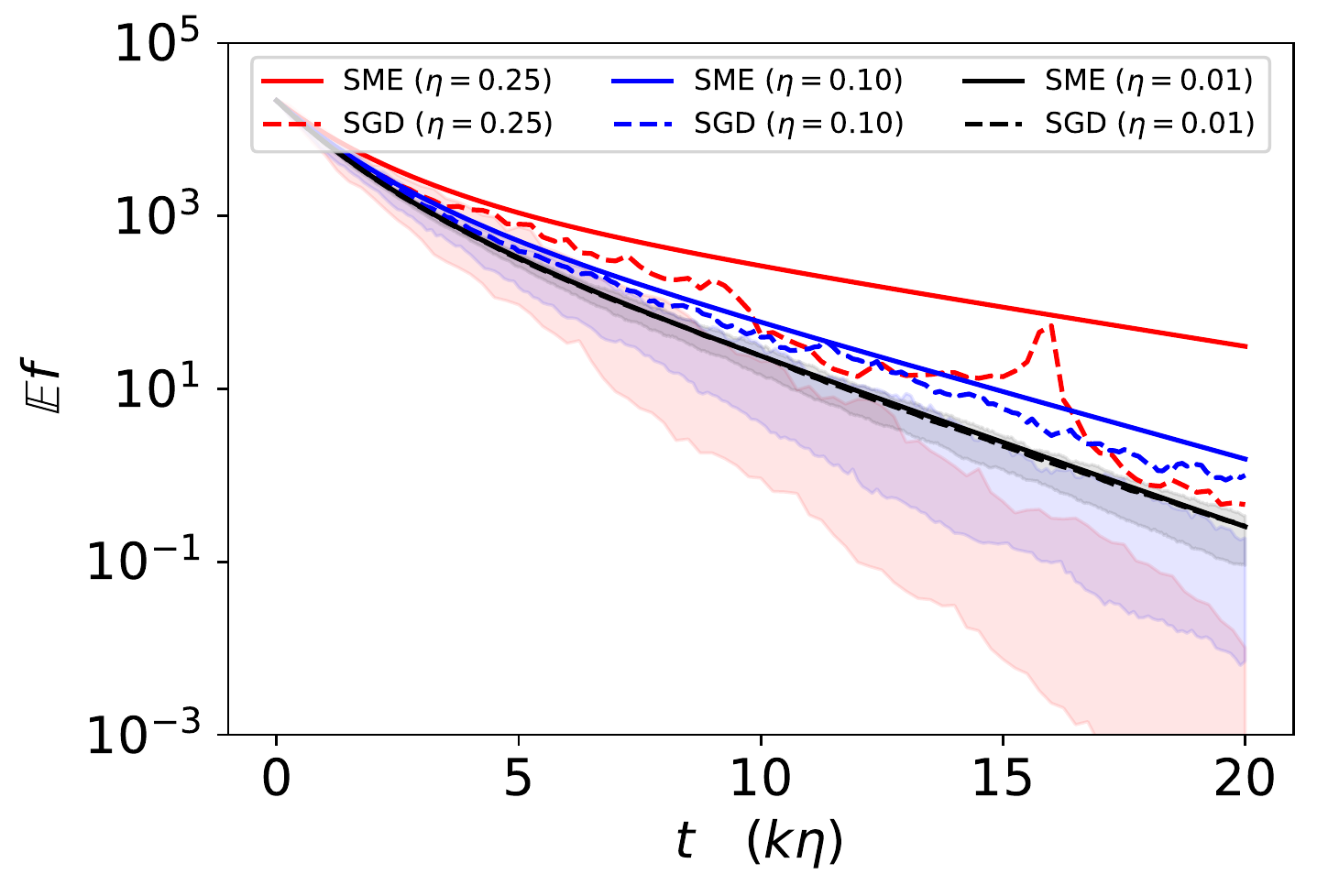}}

		\subfloat[]{\includegraphics[width=7cm]{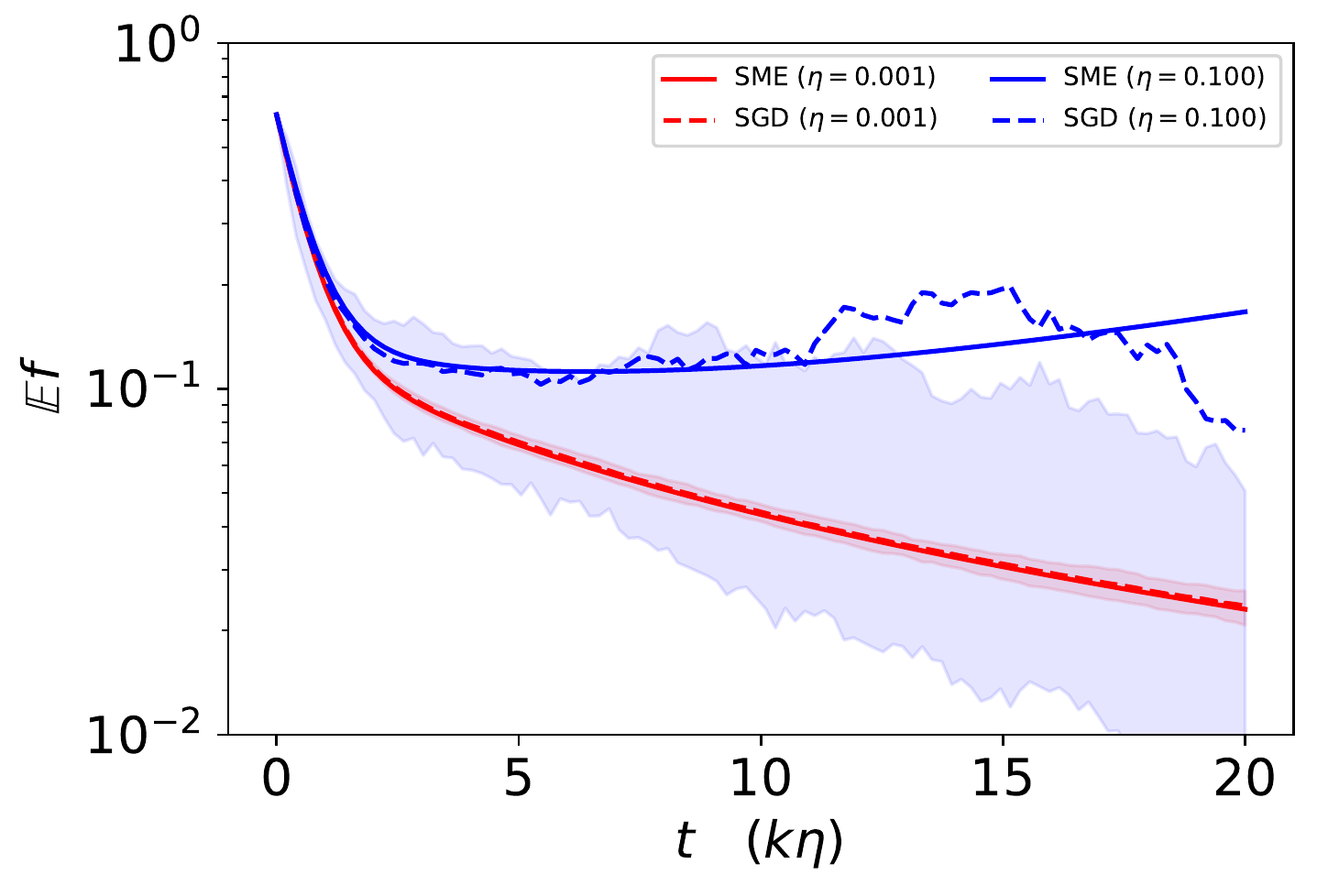}}
	\end{center}
	\caption{SME prediction vs SGD dynamics for the model variant~\eqref{eq:model_2}. (a) Order of convergence of the SME to the SGD. We use the same setup as in Fig.~\ref{fig:sgd_conv_cond}(a). Observe that our analysis again predicts the correct rate of weak error decay as $\eta$ decreases. (b) SGD paths vs order-1 SME prediction. Solid lines are SME exact solutions and dotted lines are means of SGD paths over 500 runs, and the 25-75 percentiles are shaded. We observe convergence of $\E f$ at constant $\eta$, and that the sample mean is dominated by few large values, as observed by the deviation of the percentiles from the mean.
	(b) Variance-induced explosion. As predicted by the SME analysis, if $\eta > 2\lambda_d(H)$ (Here, $\lambda_d(H) = 0.01$), variance-induced instability sets in.}
	\label{fig:sgd_bs_conv_cond}
\end{figure}

\subsection{SME analysis of MSGD}
Let us now use the SME to analyze MSGD applied to model~\eqref{eq:model_1}.
We have shown earlier that $\Sigma(x) = H$. Thus, according to Thm.~\ref{thm:sme_msgd}, the order-1 SME for MSGD is
\begin{align}\label{eq:sme_msgd_mod1}
	\begin{split}
		dV_t &= - [ \mu V_t + H X_t ] dt + \sqrt{\eta} H dW_t, \\
		dX_t &= V_t dt,
	\end{split}
\end{align}
with $X_0=x_0$ and $V_0=0$. If we set $Y_t := (V_t, X_t) \in \R^{2d}$, $U_t$ a $2d$-dimensional Brownian motion with first $d$ coordinates equal to $W_t$, and define block matrices
\begin{align}\label{eq:ABdef}
	A := \begin{pmatrix}
		\mu I &   H \\
		-I    &   0
	\end{pmatrix},
	\qquad
	B := \begin{pmatrix}
		H & 0 \\
		0 & 0
	\end{pmatrix},
\end{align}
we can then write~\eqref{eq:sme_msgd_mod1} as
\begin{align*}
	dY_t = - A Y_t + \sqrt{\eta} B dU_t, \qquad Y_0 = (0, x_0),
\end{align*}
which admits the explicit solution
\begin{align*}
	Y_t = e^{ - At}
	\left(
		Y_0 + \sqrt{\eta} \int_{0}^{t} e^{A s} B dU_s.
	\right).
\end{align*}
By It\^{o}'s isometry, we have
\begin{align}\label{eq:msgd_mod1_nonexactf}
	\E f(X_t)
	=& \tfrac{1}{2}
	\left[
		| \diag(0, H)^{\nicefrac{1}{2}} e^{-At} {Y_0} |^2
		+ \eta \int_{0}^{t} |\diag(0, H)^{\nicefrac{1}{2}} e^{-(t-s)A} B|^2 ds
	\right],
\end{align}
One can see immediately that a similar two-phase behavior is present, but the property of the descent phase now hinges on the spectral properties of the matrix $A$ (instead of $H$). Before proceeding, we first observe that the eigenvalues of $A$ can be written as
\begin{align}\label{eq:eigvals_sme_msgd}
	\lambda(A) := \{ \Lambda_+, \Lambda_- \},
	\qquad
	\Lambda_{\pm,i} = \tfrac{1}{2} {\left( \mu \pm \sqrt{\mu^2 - 4 \lambda_i(H)}\right)},
	\qquad
	i=1,2,\dots,d.
\end{align}
In particular, $\Re \lambda_i(A) > 0$ for all $i$ as long as $\mu>0$. We also need the following simple result concerning the decay of the norm of $e^{-tA}$ if all eigenvalues of $A$ have positive real part.
\begin{lemma}
	\label{lem:decay_estimate_A}
	Let $A$ be a real square matrix such that all eigenvalues have positive real part. Then,
	\begin{enumerate}[(i)]
		\item For each $\epsilon >0$, there exists a constant $C_\epsilon > 0$ independent of $t$ but depends on $\epsilon$, such that
		\begin{align*}
			| e^{-t A} | \leq C_\epsilon e^{- t (\min_{i} \Re \lambda_i(A) - \epsilon)}
		\end{align*}
		\item If in addition $A$ is diagonalizable, then there exists a constant $C>0$ independent of $t$ such that
		\begin{align*}
			| e^{-t A} | \leq C e^{- t \min_{i} \Re \lambda_i(A)}
		\end{align*}
	\end{enumerate}
\end{lemma}
\begin{proof}
	See Appendix~\ref{app:sec:computation}.
\end{proof}

With the above results, we can now characterize the decay of the objective under momentum SGD.
From expression~\eqref{eq:eigvals_sme_msgd}, we see that as long as $\mu^2 \neq 4\lambda_i$ for any $i=1,\dots,d$, $A$ has $2d$ distinct eigenvalues and is hence diagonalizable. We shall hereafter assume that $\mu$ is in general position such that this is the case.
Using Lem.~\ref{lem:decay_estimate_A} and expression~\eqref{eq:msgd_mod1_nonexactf}, we arrive at the estimate
\begin{align}\label{eq:msgd_mod1_bdf}
	\E f(X_t) \leq& \tfrac{1}{2} C^2 |x_0|^2 \lambda_1(H) e^{- 2 t \min_{i} \Re \lambda_i(A)}
	+ \tfrac{1}{2} \tfrac{\eta C^2 {\lambda_1(H)}^3}{\min_i \Re \lambda_i(A)}
	(1 - e^{- 2 t \min_i \Re \lambda_i(A)}).
\end{align}
This result tells us that the convergence rate of the descent phase is now controlled by the minimum real part of the eigenvalues of $A$, instead of the minimum eigenvalue of $H$. In particular, we achieve the best linear convergence rate by maximizing the smallest real part of the eigenvalues of $A$. This leads to the following optimization problem for the optimal convergence rate:
\begin{align*}
	\sup_{\mu\in(0,\infty)}
	\min_{i=1,\dots,d}
	\min_{s\in\{+1, -1\}}
	\left \{
		\Re \left[ \mu + s \sqrt{\mu^2 - 4\lambda_i(H) } \right]
	\right \}
\end{align*}
Since $H$ is positive definite, the supremum is attained at $\mu^*=2\sqrt{\lambda_d(H)}$ with the rate also equal to $2 \sqrt{\lambda_d(H)}$.
However, note that if we take $\mu = \mu^*$ exactly, one can check that $A$ is no longer diagonalizable and by Lem.~\ref{lem:decay_estimate_A}, the rate is slightly diminished, thus technically we can take $\mu$ as close to $\mu^*$ as we like (i.e.\,the rate is as close to $2 \sqrt{\lambda_d(H)}$ as we like), but exact equality is not technically deducible from current results. In Fig.~\ref{fig:msgd_conv_cond}(c), we demonstrate the optimal choice of $\mu$ and its effect on the convergence rate. Moreover, observe that as $\mu$ increases, the number of complex eigenvalues start to decrease, and the magnitudes of the imaginary parts of the complex eigenvalues also decrease. This signifies that increasing $\mu$ causes oscillations to decreases in magnitude and frequency. This is again corroborated by numerical experiments (Fig.~\ref{fig:msgd_conv_cond}(c)).
\begin{figure}[h]
	\begin{center}
		\subfloat[]{\includegraphics[width=7cm]{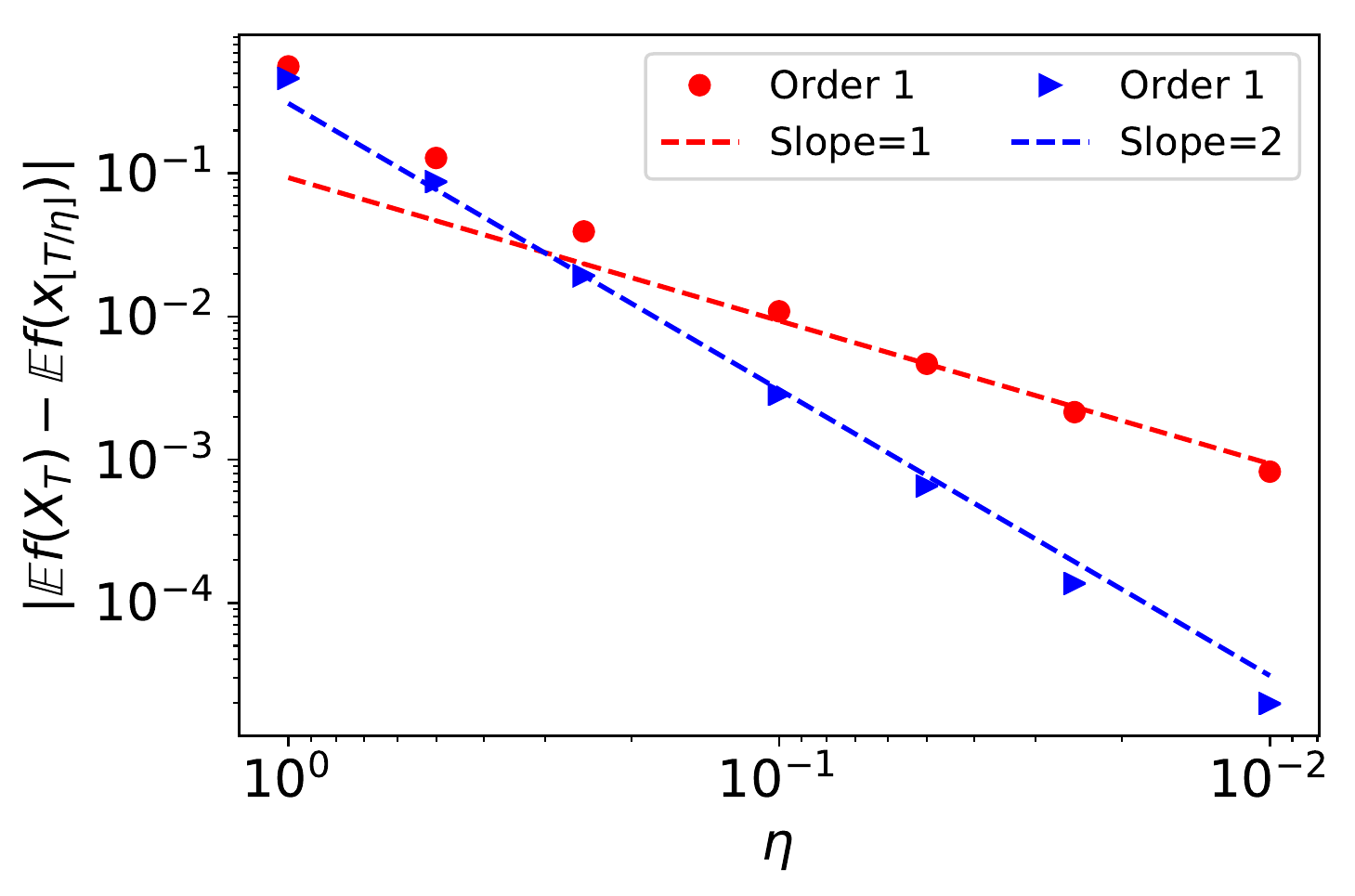}}
		\subfloat[]{\includegraphics[width=7cm]{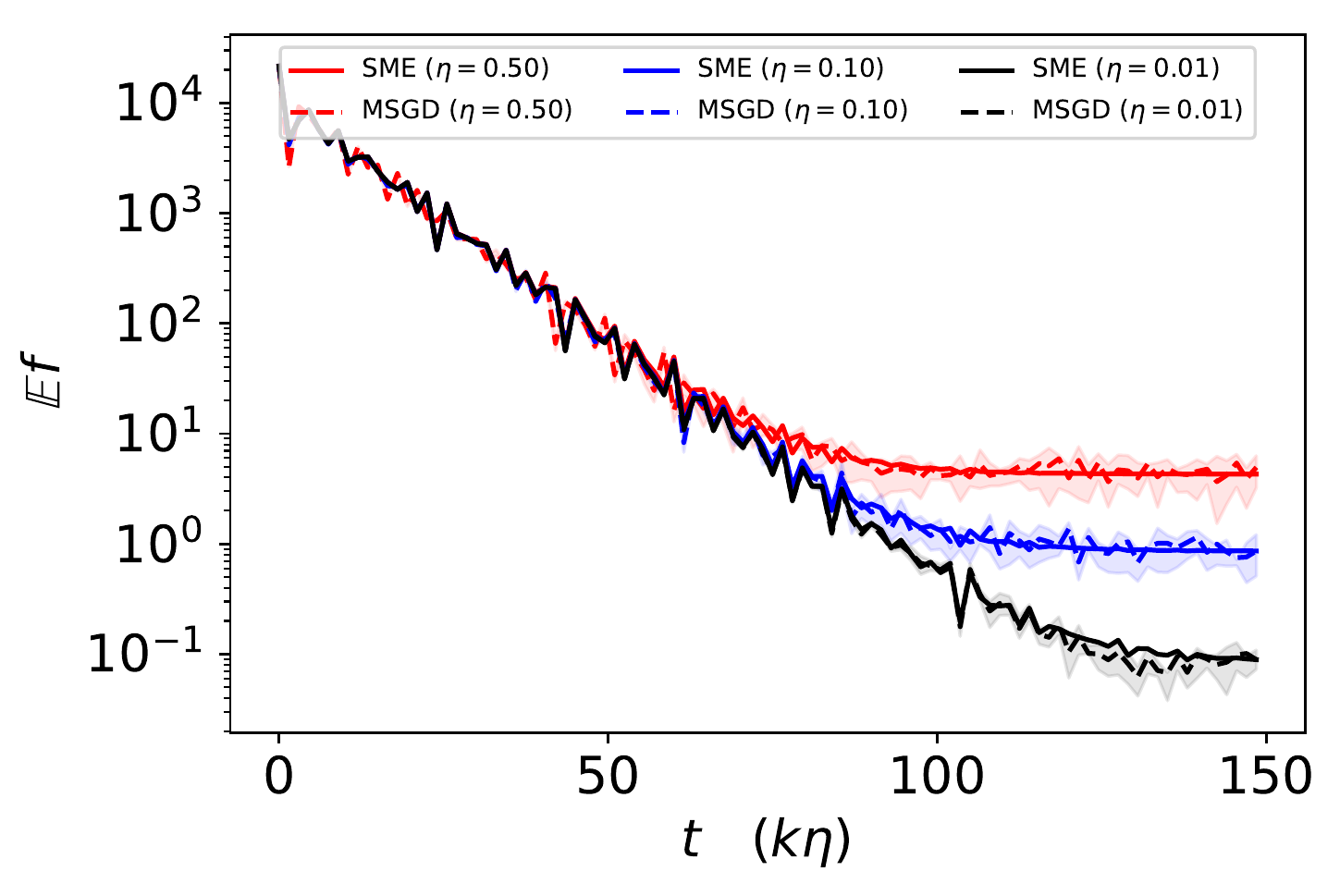}}

		\subfloat[]{\includegraphics[width=7cm]{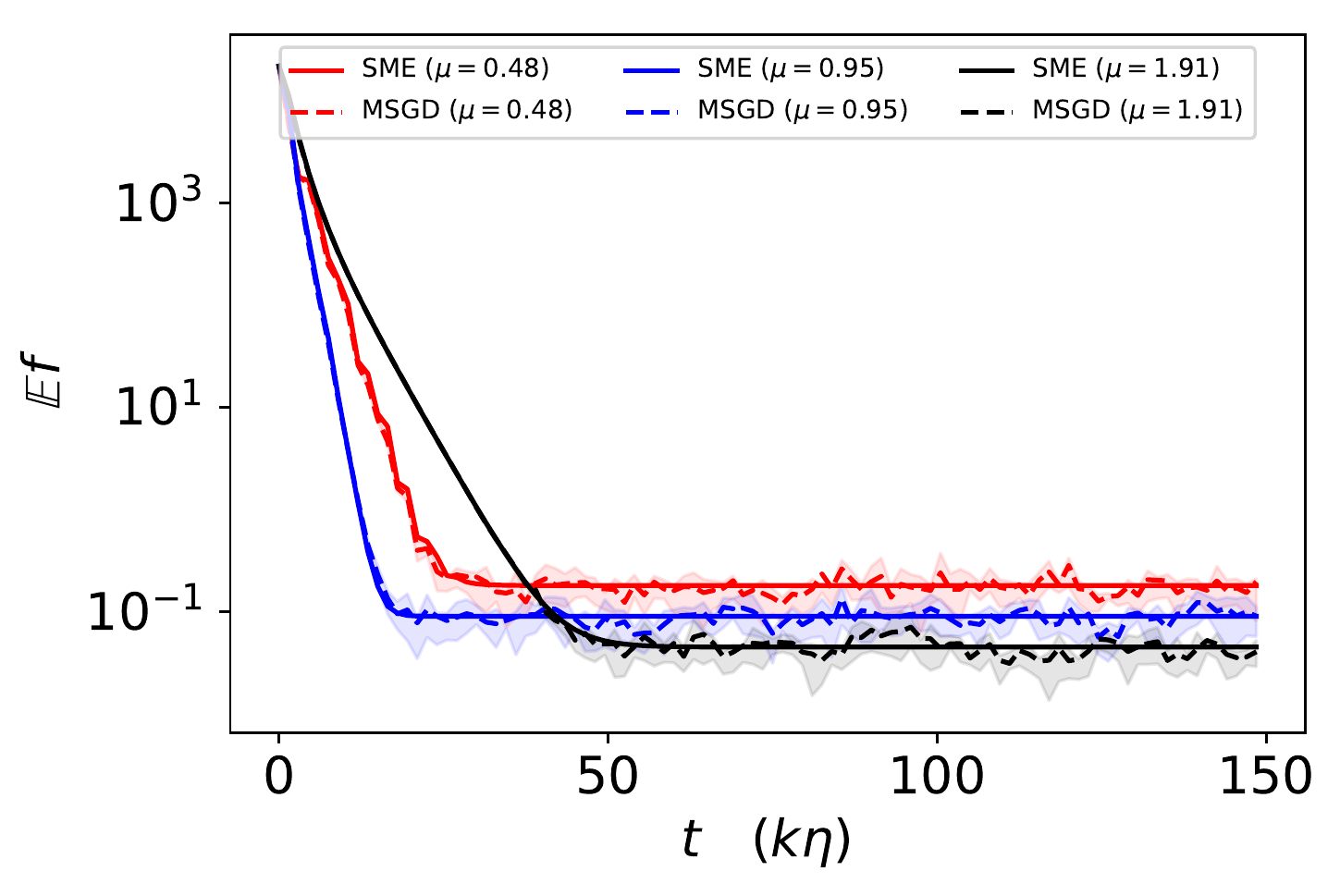}}
		\subfloat[]{\includegraphics[width=7cm]{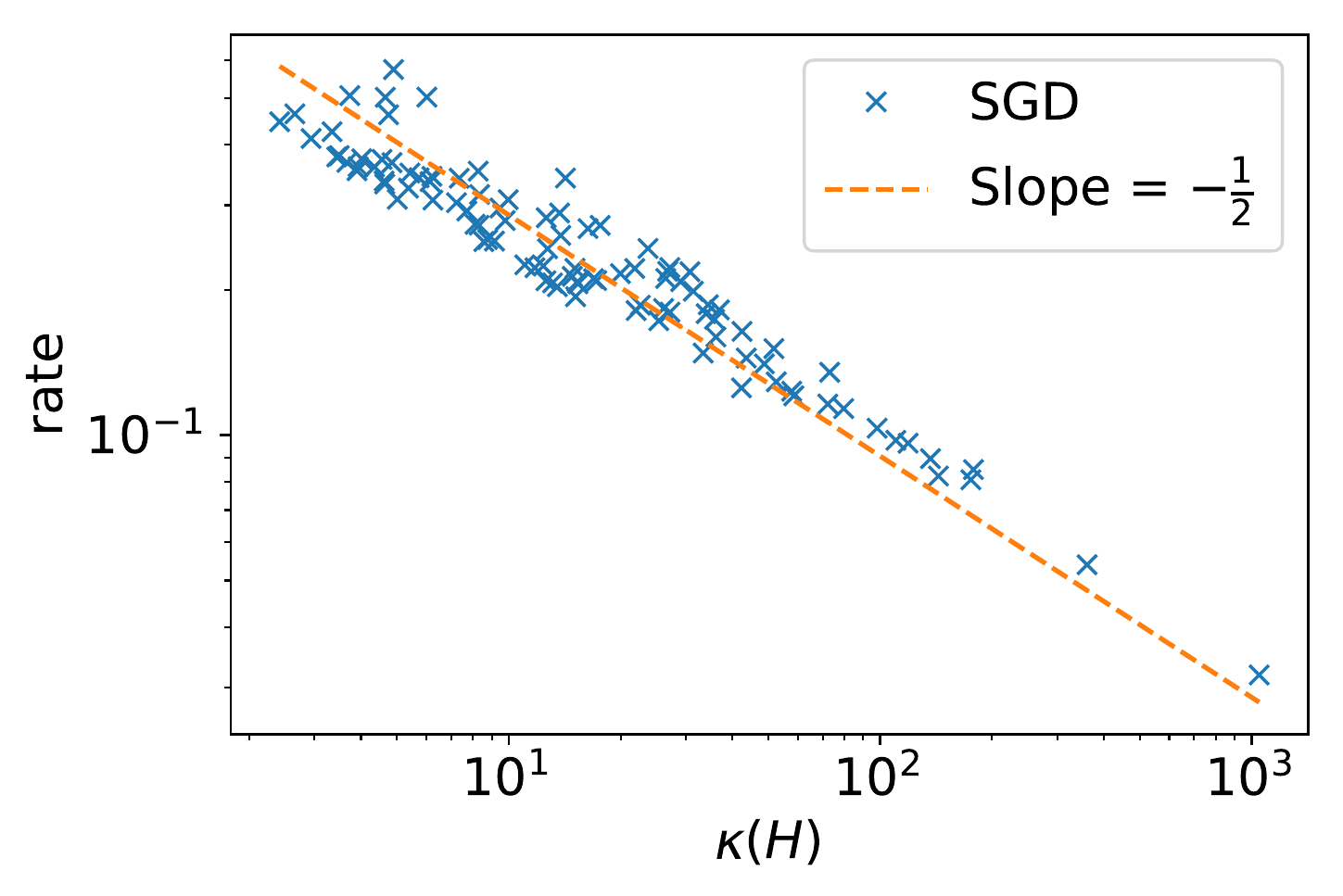}}
	\end{center}
	\caption{SME prediction vs MSGD dynamics. (a) and (b) SME vs MSGD dynamics at $\mu=0.1$ for different learning rates $\eta$. As before, the SME prediction gets better as $\eta$ decreases according to the predicted order. Notice also the presence of oscillations, due to the complex eigenvalues of $A$. (c) Optimal descent rate of the SGD is achieved by the SME prediction $\mu=\mu^*$, which is $0.95$ in this case. Notice that exactly as predicted by the SME, increasing $\mu$ decreases the oscillation frequency and magnitude (due to having fewer complex eigenvalues and smaller imaginary parts), as well as the asymptotic fluctuations (due to formula~\eqref{eq:msgd_mod1_exactf}). (d) Descent rate vs condition number. $H$ is generated with different condition numbers, and the descent rate of MSGD is $\propto {\kappa(H)}^{-1/2}$, as predicted by the SME, which for badly conditioned $H$ gives a large improvement.}
	\label{fig:msgd_conv_cond}
\end{figure}

Another interesting observation is that by the identification $t=\eta k$, the descent rate (in terms of $k$) is $2\sqrt{\lambda_d(H)}\eta$. As before, if we choose the maximal stable learning rate we would have $\hat{\eta} \propto 1/\lambda_1(H)$ ($\hat{\eta} = \eta^2$ according to the scaling introduced in~\eqref{eq:rescale}). Thus, for the MSGD iterates we have its descent rate $\propto {\kappa(H)}^{-1/2}$, which is a huge improvement over SGD, whose rate is $\propto \kappa(H)^{-1}$, especially for badly conditioned matrices where $\kappa(H) \gg 1$. In Fig.~\ref{fig:msgd_conv_cond}(d), we plot the MSGD initial descent rates for varying condition numbers of $H$. Again, we observe that the SME analysis gives the correct characterization of the precise dynamics and recovers the square-root relationship with condition number.

Finally, let us discuss the effect of adding momentum to the asymptotic fluctuations due to noisy gradients. Note that it is not correct to conclude, using Eq.~\eqref{eq:msgd_mod1_bdf}, that taking $\mu\approx\mu^*$ also gives the lowest fluctuations. This is because the constant $C$ depends on $\mu$ as well, as is evidenced in the proof of Lem.~\ref{lem:decay_estimate_A}, which shows that $C$ depends on the conditioning of the eigenvector matrix of $A$. To proceed, we do not use the bounds~\eqref{eq:msgd_mod1_bdf}. Instead, we explicitly diagonalize $A$ and after some computations, we arrive at the exact expression for $\E f(X_t)$
\begin{align}\label{eq:msgd_mod1_exactf}
	\E f(X_t)
	=& \tfrac{1}{2}
	| \diag(0, H)^{\nicefrac{1}{2}} e^{-At} {Y_0} |^2 \\
	&+ \tfrac{1}{2}\eta
	\sum_{i=1}^{d}
	\tfrac{\lambda_i^3}{|\mu^2 - 4 \lambda_i|}
	\left[
		\tfrac{1 - e^{-2t\Re \Lambda_{+,i}}}{2 \Re \Lambda_{+,i}}
		+ \tfrac{1 - e^{-2t\Re \Lambda_{-,i}}}{2 \Re \Lambda_{-,i}}
		- 2 R(t,\mu, \lambda_i(H))
	\right]
\end{align}
where
\begin{align}\label{eq:msgd_mod1_exactf_Rdef}
	R(t,\mu, \lambda) =
	\begin{cases}
		\tfrac{1 - e^{-t\mu}}{\mu}
		& \mu \geq 2 \sqrt{\lambda} \\
		\tfrac{\mu +\sqrt{4 \lambda -\mu ^2} e^{- \mu t} \sin (t \sqrt{4 \lambda -\mu ^2})-\mu  e^{- \mu t} \cos (t \sqrt{4 \lambda -\mu ^2})}{4 \lambda }
		& \mu < 2 \sqrt{\lambda}
	\end{cases}.
\end{align}
In particular, the asymptotic loss value induced by noise is
\begin{align}\label{eq:msgd_sme_asympnoise}
	\lim_{t\rightarrow \infty} \E f(X_t)
	=& \tfrac{1}{2}\eta
	\sum_{i=1}^{d}
	\tfrac{\lambda_i(H)^3}{|\mu^2 - 4 \lambda_i(H)|}
	\left[
		\tfrac{1}{2 \Re \Lambda_{+,i}}
		+ \tfrac{1}{2 \Re \Lambda_{-,i}}
		- 2 \min \left \{ \tfrac{\mu}{4\lambda_i(H)} , \tfrac{1}{\mu} \right \}
	\right]
\end{align}
Observe that this function (in fact, each term in the sum) is monotone-decreasing in $\mu$, and for $\mu \ll 1$ it scales like $\mu^{-1}$, and for $\mu \gg 1$ it scales like $\mu^{-3}$. Thus, increasing the momentum parameter decreases the asymptotic noise in the iterates, i.e. decreases the asymptotic value of $\E f$, which should be 0 in the absence of noise. This again agrees with the actual MSGD dynamics (Fig.~\ref{fig:msgd_conv_cond}(b)). Consequently, to obtain ``optimal'' tradeoff between descent and noise, we would like a momentum schedule that equals $\mu^*$ in the descent phase and increases to infinity (in the original scaling this corresponds to $\hat{\mu}\rightarrow 0$) as we approach the optimum. Finding this optimal schedule can be cast as an optimal control problem~\citep{li2017stochastic}, and a rigorous investigation of these approaches will be considered in subsequent work.

\subsection{SME analysis of SNAG}

Finally, let us see what we can say, using the SME approach, about the difference between MSGD and SNAG in this stochastic setting. Let us first consider the case of constant momentum. From Thm.~\ref{thm:sme_nag}, we know that the order-1 SMEs are identical, so we must consider higher order SMEs. A straightforward computation yields the following order-2 SMEs for MSGD and SNAG (again we let $Y_t=(V_t,X_t)$)
\begin{align*}
	&\text{MSGD:} & &dY_t = - A_{1} Y_t + \sqrt{\eta}   B dU_t, \qquad Y_0 = (0, x_0), \\
	&\text{SNAG:} & &dY_t = - A_{2} Y_t + \sqrt{\eta}   B dU_t, \qquad Y_0 = (0, x_0),
\end{align*}
where $A_{i} = A + \tfrac{1}{2} \eta E_{i}$ with $A,B$ as defined in~\eqref{eq:ABdef} and
\begin{align*}
	E_1 := \begin{pmatrix}
		\mu^2 I - H &  \mu H \\
		\mu I    &   H
	\end{pmatrix},
	\qquad
	E_2 := \begin{pmatrix}
		\mu^2 I + H &  \mu H \\
		\mu I    &   H
	\end{pmatrix}.
\end{align*}
From the analysis in Sec.~\ref{sec:sme_msgd}, the descent rate is dominated by the minimal real parts of the eigenvalues of $A_{i}$, which are respectively
\begin{align*}
	&\lambda(A_1) =
	\left \{
		\tfrac{1}{4} \left(\mu  (\eta  \mu +2) \pm \sqrt{\mu ^2 (\eta  \mu +2)^2+4 \eta ^2 {\lambda_i(H)}^2-8 {\lambda_i(H)} (\eta  \mu +2)}\right),
		\quad
		i=1,\dots,d
		\right \} \\
		&\lambda(A_2) =
		\left \{
			\tfrac{1}{4} \left(
				\mu  (\eta  \mu +2)
				+2 \eta  {\lambda_i(H)}
				\pm \sqrt{\eta  \mu +2} \sqrt{\mu ^2 (\eta  \mu +2)+4 {\lambda_i(H)} (\eta \mu -2)}
			\right),
			\quad
			i=1,\dots,d
		\right \}
\end{align*}
We observe that for small $\mu$ (i.e.\,$\hat{\mu}\approx 1$ in the usual MSGD scaling), the terms in square-roots are negative and hence for the same small $\mu$, the convergence rate of SNAG is $\tfrac{1}{2} \eta \lambda_d(H)$ larger than that of MSGD. This says in particular that for $H$ with larger $\lambda_d(H)$, the acceleration is more pronounced.
Moreover, recall that the asymptotic fluctuations is given by
\begin{align*}
	\eta \lim_{t\rightarrow\infty}\int_{0}^{t} |\diag(0, H)^{\nicefrac{1}{2}} e^{-(t-s)(A + \tfrac{1}{2}\eta E_i)} B|^2 ds.
\end{align*}
Without performing tedious computations, we can see that since $E_2 - E_1$ is positive definite, the exponential for the SNAG case decays more rapidly, and hence the eventual fluctuations are expected to be lower. These observations from the SME are again consistent with the behavior of their SGA counter-parts, as shown in Fig.~\ref{fig:msgd_vs_snag}(a).
On the other hand, if we pick $\mu$ for each case by separately maximizing the smallest real part of the eigenvalues (as in Sec.~\ref{sec:sme_msgd}), we obtain similar convergence rates up to $\eta^2$. In other words, if we tune $\mu$ well, there would be no difference between MSGD and SNAG in terms of descent rate (Fig.~\ref{fig:msgd_vs_snag}(b)).
\begin{figure}[h]
	\begin{center}
		\subfloat[]{\includegraphics[width=15cm]{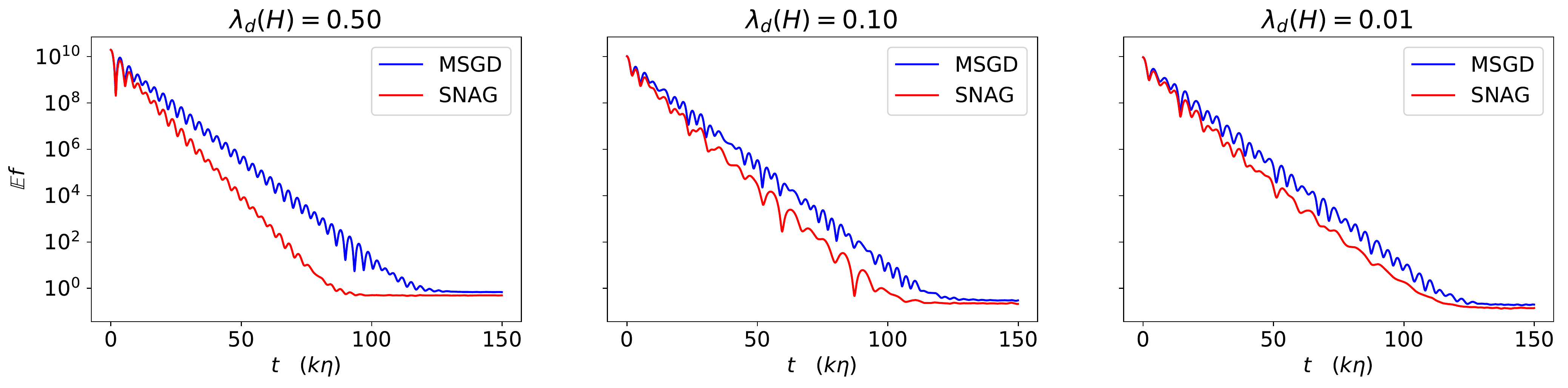}}

		\subfloat[]{\includegraphics[width=15cm]{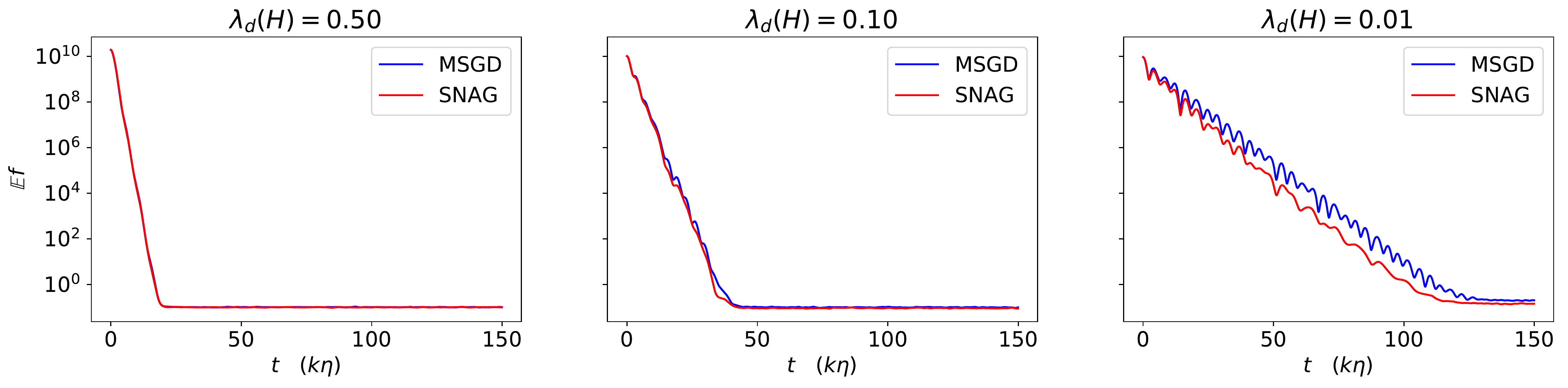}}
	\end{center}
	\caption{MSGD vs SNAG (with constant momentum) dynamics for $\eta=0.1$ and different $\lambda_d(H)$. (a) Dynamics at fixed $\mu=0.2$. We observe that as predicted by the SME analysis, SNAG enjoys a faster linear convergence rate in the descent phase, as well as lower asymptotic fluctuations. The acceleration is indeed more pronounced for larger $\lambda_d(H)$. (b) When, $\mu$ for each case is chosen optimally for the descent (by maximizing the minimal real part of the eigenvalues of $A_1,A_2$ respectively), the dynamics becomes similar. }
	\label{fig:msgd_vs_snag}
\end{figure}

Now, let us discuss the varying momentum case. According to~\eqref{eq:snag_dyn_sme_order1}, for some small $t_0 > 0$ we have the order-1 SME for $t\in[t_0,T]$
\begin{align*}
	dY_t = - A_t Y_t + \sqrt{\eta}   B dU_t, \qquad Y_{t_0} = (v_{t_0}, x_{t_0})
	\qquad
	A_t := \begin{pmatrix}
		\tfrac{3}{t} I &   H \\
		-I    &   0
	\end{pmatrix},
\end{align*}
and $B$ is defined as in~\eqref{eq:ABdef}. This admits the explicit solution
\begin{align*}
	Y_t = e^{ - (t-t_0) \tilde{A}_t}
	\left(
		Y_{t_0} + \sqrt{\eta} \int_{t_0}^{t} e^{s \tilde{A}_s} B dU_s.
	\right),
	\qquad
	\tilde{A}_t := \begin{pmatrix}
		3 \tfrac{\log(t/t_0)}{t-t_0}I &   H \\
		-I    &   0
	\end{pmatrix}.
\end{align*}
The eigenvalues of $\tilde{A}_t$ are
\begin{align*}
	\lambda(\tilde{A}_t) =
	\left \{
		\tfrac{1}{2}
		\left(
			3 \tfrac{\log(t/t_0)}{t-t_0} \pm
			\sqrt{
				9 {[\tfrac{\log(t/t_0)}{t-t_0}]}^2 - 4 \lambda_i(H)
			}
		\right),
		\qquad
 		i=1,\dots,d
	\right \}.
\end{align*}
Since there is no lower-bound on the minimal real part, the convergence is sub-linear. This is expected because the $\mathcal{O}(1/t)$ momentum schedule is suited for non-strongly-convex functions, whereas constant momentum is more appropriate for strong-convex functions~\citep{nesterov2013introductory}. Furthermore, we observe that since the real parts of all eigenvalues of $\tilde{A}_t$ converge to 0 as $t\rightarrow \infty$, according to the analysis in Sec.~\ref{sec:sme_msgd}, the asymptotic fluctuations due to noise should be large. Fig.~\ref{fig:msgd_vs_snag_dyn} confirms these observations and further suggests that in the case of stochastic gradient methods, more careful momentum schedules must be derived in order to balance descent and fluctuations, e.g.\,using the optimal control framework presented in~\cite{li2017stochastic}.
\begin{figure}[h]
	\begin{center}
		\includegraphics[width=15cm]{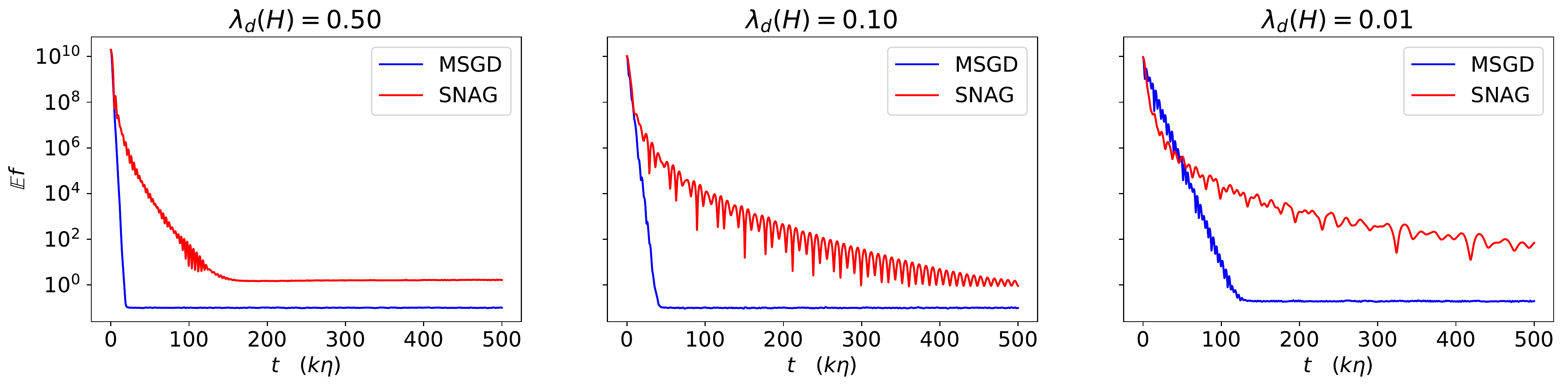}
	\end{center}
	\caption{MSGD vs SNAG (with dynamic momentum according to Nesterov's choice~\eqref{eq:nesterov_mu}) dynamics for $\eta=0.1$ and different $\lambda_d(H)$. We see that the convergence is indeed sub-linear, and moreover, the asymptotic fluctuations are large compared with MSGD, in which case $\mu$ here is picked to achieve optimal descent rate.}
	\label{fig:msgd_vs_snag_dyn}
\end{figure}

\section{Conclusion}
\label{sec:conclusion}

In this paper, we developed the general mathematical foundation of the stochastic modified equations framework for analyzing stochastic gradient algorithms. In particular, we demonstrate that this approach is (1) rigorous, (2) flexible and (3) useful. Indeed, the technique of weak approximations provides a precise mathematical framework for analyzing the relationship between stochastic gradient algorithms and  stochastic differential equations. This should be contrasted with strong approximations in the numerical analysis of SDEs, where approximations are required to hold path-wise, say in the mean-square sense~\citep{Kloeden1992}. The weak formulation greatly increases the flexibility of modelling different type of stochastic gradient algorithms, as we have demonstrated in Sec.~\ref{sec:approx_theorems}. In fact, the main result relating discrete-time algorithms and continuous-time SDEs (Thm.~\ref{thm:one_step_to_multi_step}) is proved in a fairly general setting that allows one to derive a variety of SMEs for different variations of the SGAs (Sec.~\ref{sec:approx_theorems}). Finally, in Sec.~\ref{sec:applications}, we demonstrated the usefulness of the SME approach through explicit calculations. This is enabled by the precise approximation nature of the SMEs and the application of stochastic calculus tools. In particular, we uncovered interesting behaviors of SGAs when applied to a simple yet non-trivial setting, including the tradeoff of descent and fluctuations, the relationship with condition numbers and the subtle differences of MSGD and SNAG in the stochastic setting. In the subsequent work in the series, we will focus on applications, where we extend the SME formalism to study adaptive algorithms and related topics.

\newpage
\appendix

\section{General existence, uniqueness and moment estimates for SDEs}
\label{app:sec:general_SDE_thms}

In this section, we establish general existence, uniqueness and moment estimates for the stochastic differential equations that we encounter in this paper. The results here will be used throughout the subsequent proofs. We note that although similar well-posedness results are well-known, here we require a slightly more general version (where the drift and diffusion terms are themselves random functions) in order to deal with the analysis in Appendix~\ref{app:sec:deriv_x}. Moreover, we need uniform estimates with respect to parameters ($\eta,\epsilon$), which warrants the following standard but necessary analysis.

Let $T>0$ and $Q$ be a subset of a Euclidean space. For $(x,t,q) \in  \R^d \times [0,T] \times Q$, let $B(x,t,q)$ be a $d$-dimensional random vector and $S(x,t,q)$ be a $d\times d$-dimensional random matrix. Throughout this section we assume:
\begin{assumption}
\label{app:assu:B_and_S}
	The random functions $B, S$ satisfy the following:
	\begin{enumerate}[(i)]
		\item $B ,S $ are $W_t$-adapted and continuous in $(x,t) \in \R^d \times [0,T]$ almost surely
		\item $B ,S $ satisfy a uniform linear growth condition, i.e.\,there exists a non-random constant $L>0$ such that
		\begin{align*}
			| B (x,t,q) |^2
			+
			| S (x,t,q) |^2
			\leq L^2(1 + \vert x \vert^2) \quad a.s.
		\end{align*}
		for all $x,y\in\R^d$, $t\in [s,T]$, $q \in Q$.
		\item $B ,S $ satisfy a uniform Lipschitz condition in $x$, i.e.
		\begin{align*}
			| B (x,t,q) - B (y,t,q) |
			+
			| S (x,t,q) - S (y,t,q) |
			\leq L| x - y | \quad a.s.
		\end{align*}
		for all $t\in [s,T]$, $q \in Q$.
	\end{enumerate}
\end{assumption}

\begin{theorem}
\label{app:thm:exist_uniq}
	Let $s \in [0,T)$ and for each $q\in Q$, let $\{ \phi^q_t : t\in[s,T]\}$ be a $\R^d$-valued, $W_t$-adapted random process that is continuous in $t\in [s,T]$ almost surely, with
	\begin{align}
	\label{eq:app:thm:exist_uniq:phi_moment_cond}
		\sup_{q \in Q} \E\sup_{t\in[s,T]} | \phi^q_t |^2 < \infty.
	\end{align}
	Then, for each $q\in Q$ the stochastic differential equation
	\begin{align}
	\label{eq:app:thm:exist_uniq:main_sde}
		\xi^q_t = \phi^q_t
		+ \int_{s}^{t} B (\xi^q_v,v,q) dv
		+ \int_{s}^{t} S (\xi^q_v,v,q) dW_v
	\end{align}
	admits a unique solution $\{\xi^q_t:t\in[s,T]\}$ which is continuous for $t\in [s,T]$ a.s.\,and satisfies
	\begin{align}
	\label{eq:app:thm:exist_uniq:main_estimate}
		\sup_{q \in Q} \E \sup_{t\in[s,T]} | \xi^q_t |^2
		\leq
		C
		\left(
			1 + \sup_{q \in Q} \E \sup_{t\in[s,T]} | \phi^q_t |^2
		\right)
	\end{align}
	for some constant $C>0$ that depends only on $L,T$.
\end{theorem}
\begin{proof}
	For each $q\in Q$, let us define the recursion
	\begin{align*}
		\xi^{q,0}_t &= \phi^q_t \\
		\xi^{q,m+1}_t &= \phi^q_t + \int_{s}^{t} B (\xi^{q,m}_v,v,\eta) dv
		+ \int_{s}^{t} S (\xi^{q,m}_v,v,\eta) dW_v, \quad m\geq 0.
	\end{align*}
	Note that Assumption~\ref{app:assu:B_and_S} implies each $\xi^{q,m}_t$ is well-defined. Now, let $m\geq 1$. By It\^{o}'s isometry, we have
	\begin{align}
	\label{eq:app:thm:exist_uniq:m_recurs}
		| \xi^{q,m+1}_t - \xi^{q,m}_t |^2
		\leq& 2
		\left|
			\int_{s}^{t}
			B (\xi^{q,m}_v, v, q) - B (\xi^{q,m-1}_v, v, q)
			dv
		\right|^2 \\
		&+ 2
		\left|
			\int_{s}^{t}
			S (\xi^{q,m}_v, v, q) - S (\xi^{q,m-1}_v, v, q)
			dW_v
		\right|^2 \\
		\leq& 2T
		\int_{s}^{t}
		|
			B (\xi^{q,m}_v, v, q) - B (\xi^{q,m-1}_v, v, q)
		|^2
		dv \\
		&+ 2
		\int_{s}^{t}
		|
			S (\xi^{q,m}_v, v, q) - S (\xi^{q,m-1}_v, v, q)
		|^2
		dv.
	\end{align}
	Thus, applying the Lipschitz assumption~\ref{app:assu:B_and_S} (iii) and taking expectations, we get
	\begin{align}
	\label{eq:app:thm:exist_uniq:m_iter}
		\E | \xi^{q,m+1}_t - \xi^{q,m}_t |^2
		\leq 2L^2(1+T) \int_{s}^{t}
		\E | \xi^{q,m}_v - \xi^{q,m-1}_v |^2 dv.
	\end{align}
	Now, for $m=0$, Assumption~\ref{app:assu:B_and_S} (ii) together with~\eqref{eq:app:thm:exist_uniq:phi_moment_cond} gives
	\begin{align}
		\label{eq:app:thm:exist_uniq:m_0}
		\E | \xi^{q,1}_t - \xi^{q,0}_t |^2
		\leq C \int_{s}^{t}
		\left(
			1 + \sup_{q\in Q} \E | \phi^q_v |^2
		\right) dv
		\leq C' (t - s).
	\end{align}
	Combining~\eqref{eq:app:thm:exist_uniq:m_iter} and~\eqref{eq:app:thm:exist_uniq:m_0}, we have
	\begin{align}
	\label{eq:app:thm:exist_uniq:2_norm_est}
		\E | \xi^{q,m+1}_t - \xi^{q,m}_t |^2
		\leq
		\tfrac{{[C(t-s)]}^{m+1}}{(m+1)!}, \qquad m\geq 0
	\end{align}
	for some $C>0$ that only depends on $T$, $L$ and $C_\phi := \sup_{q \in Q} \E\sup_{t\in[s,T]} | \phi^q_t |^2$.
	Moreover, Eq.~\eqref{eq:app:thm:exist_uniq:m_recurs} implies
	\begin{align*}
		\E \sup_{t\in [s, T]} | \xi^{q,m+1}_t - \xi^{q,m}_t |^2
		\leq&
		2L^2T \int_{s}^{T} \E | \xi^{q,m}_t - \xi^{q,m-1}_t |^2 dt \\
		&+
		2 \E \sup_{t\in[s,T]}
		{
			\left|
				\int_{s}^{t}
					S (\xi^{q,m}_v, v, q) - S (\xi^{q,m-1}_v, v, q)
				dW_v
			\right|
		}^2.
	\end{align*}
	Estimate~\eqref{eq:app:thm:exist_uniq:2_norm_est} implies the last stochastic integral is a martingale, and hence using Doob's maximal inequality and It\^{o}'s isometry, we have
	\begin{align*}
		\E \sup_{t\in [s, T]} | \xi^{q,m+1}_t - \xi^{q,m}_t |^2
		\leq&
		2L^2T \int_{s}^{T} \E | \xi^{q,m}_t - \xi^{q,m-1}_t |^2 dt
		+
		8L^2 \int_{s}^{T} \E | \xi^{q,m}_t - \xi^{q,m-1}_t |^2 dt \\
		\leq& 2L^2(T+4) \tfrac{C^m{T}^{m+1}}{(m+1)!}
		= 2L^2(T+4) \tfrac{C^m{T}^{m+1}}{(m+1)!}
	\end{align*}
	Applying Markov's inequality,
	\begin{align*}
		\sum_{m\geq 0} \PP
		\left[
			\sup_{t\in[s,T]} | \xi^{q,m+1}_t - \xi^{q,m}_t | > 2^{-m}
		\right]
		\leq  2L^2(T+4) \sum_{m\geq 0} 2^{2m} \tfrac{C^m{T}^{m+1}}{(m+1)!} < \infty.
	\end{align*}
	Thus, by the Borel-Cantelli lemma,
	\begin{align*}
		\PP \left[
			\sup_{t\in[s,T]} | \xi^{q,m+1}_t - \xi^{q,m}_t | > 2^{-m} \text{ infinitely often}
		\right] = 0,
	\end{align*}
	which immediately implies
	\begin{align*}
		\xi^{q,k}_t = \xi^{q,0}_t +
		\sum_{m=0}^{k-1} (
			\xi^{q,m+1}_t - \xi^{q,m}_t
		)
		\rightarrow \xi^q_t
		\quad
		a.s.
	\end{align*}
	uniformly in $t\in[s,T]$, for some limiting process $\xi^q_t$ which is necessarily continuous almost surely and $W_t$-adapted.
	Moreover, we also have convergence in $\Lcal^2(\Omega)$ uniformly in $t$. To see this, for each $k>l$ we observe that
	\begin{align*}
		\sup_{t\in[s,T]} {(\E | \xi^{k}_t - \xi^{l}_t |^2 )}^{\nicefrac{1}{2}}
		\leq&
		\sup_{t\in[s,T]} \sum_{m=l}^{k-1}
		{(\E | \xi^{q,m+1}_t - \xi^{q,m}_t |^2 )}^{\nicefrac{1}{2}} \\
		\leq&
		\sum_{m=l}^{\infty} \sqrt{2L^2(T+4) \tfrac{C^m{T}^{m+1}}{(m+1)!}}
		\overset{l\rightarrow\infty}{\longrightarrow} 0.
	\end{align*}
	And hence $\xi^{q,k}_t$ converges uniformly in $\Lcal^2(\Omega)$ to $\xi^q_t$ as $k\rightarrow\infty$ (the limit is the same as the a.s.\,limit since a sub-sequence of it must converge a.s.). This immediately implies via the Lipschitz condition and It\^{o}'s isometry that
	\begin{align*}
		& \E
		\left|
			\int_{s}^{T} B(\xi^{q,k}_t, t, \eta) - B(\xi^{q}_t, t, \eta) dt
		\right|^2
		\leq T^2 L^2
		\sup_{t\in[s,T]}
		\E \left|
			\xi^{q,k}_t - \xi^{q}_t
		\right|^2
		\rightarrow 0, \\
		& \E
		\left|
			\int_{s}^{T} S(\xi^{q,k}_t, t, \eta) - S(\xi^{q}_t, t, \eta) dW_t
		\right|^2
		\leq T L^2
		\sup_{t\in[s,T]}
		\left|
			\xi^{q,k}_t - \xi^{q}_t
		\right|^2
		\rightarrow 0.
	\end{align*}
	Thus, $\xi^{q}_t$ satisfies~\eqref{eq:app:thm:exist_uniq:main_sde}.

	We now show the estimate~\eqref{eq:app:thm:exist_uniq:main_estimate}. From Eq.~\eqref{eq:app:thm:exist_uniq:main_sde}, we have by It\^{o}'s isometry,
	\begin{align*}
		\E | \xi^{q}_t |^2
		\leq&
		3 \E | \phi^{q}_t |^2
		+ 3 \E
		\left|
			\int_{s}^{T} B(\xi^{q}_v,v,q) dv
		\right|^2
		+ 3 \E
		\left|
			\int_{s}^{t}
				S(\xi^{q}_v,v,q)
			dW_v
		\right|^2 \\
		\leq&
		3 C_\phi
		+ 3 T^2 L^2
		\int_{s}^{T}
			\E ( 1 + | \xi^{q}_v |^2 )
		dv
		+ 3 T L^2
		\int_{s}^{T}
			\E ( 1 + | \xi^{q}_v |^2 )
		dv
	\end{align*}
	Thus, by Gronwall's lemma, we have
	\begin{align}
	\label{eq:app:thm:exist_uniq:2_norm_gronwall}
		\E | \xi^q_t |^2 \leq C (1 + C_\phi)
	\end{align}
	for some $C>0$ depending only on $T,L$.
	Consequently, we have
	\begin{align}
	\label{eq:app:thm:exist_uniq:comb_1}
		\E \sup_{t\in[s,T]} | \xi^{q}_t |^2
		\leq&
		3 C_\phi
		+ 3 T^2 L^2
		\E
		\int_{s}^{T}
			1 + | \xi^{q}_s |^2
		dt \\
		&+ 3 \E \sup_{t\in[s,T]}
		\left|
			\int_{s}^{t}
				S(\xi^{q}_t,t,q)
			dW_t
		\right|^2.
	\end{align}
	Assumption~\ref{app:assu:B_and_S} (ii) and~\eqref{eq:app:thm:exist_uniq:2_norm_gronwall} implies the last stochastic integral is a martingale, and so by Doob's maximal inequality,
	\begin{align}
	\label{eq:app:thm:exist_uniq:comb_2}
		\left(
			\E \sup_{t\in[s,T]} \left|
				\int_{s}^{t} S(\xi_t,t,\eta) dW_t
			\right|
		\right)^2
		\leq& \int_{s}^{T}
		4 L^2 (1 + \E | \xi_t |^2) dt.
	\end{align}
	Combining~\eqref{eq:app:thm:exist_uniq:comb_1} and~\eqref{eq:app:thm:exist_uniq:comb_2}, we arrive at ~\eqref{eq:app:thm:exist_uniq:main_estimate}.

	Finally, we show uniqueness. Suppose that $\xi_t,\xi'_t$ are two solutions to~\eqref{eq:app:thm:exist_uniq:main_sde}. The same calculation as before shows that
	\begin{align*}
		\E | \xi_t - \xi'_t |^2
		\leq 2L^2T(1+T) \int_{s}^{t}
		\E | \xi_v - \xi'_v |^2 dv.
	\end{align*}
	and Gronwall's lemma implies
	\begin{align*}
		\E | \xi_t - \xi'_t |^2 \leq e^{2L^2T(1+T)} \E | \xi_s - \xi'_s |^2 = 0.
	\end{align*}
\end{proof}

% Moment Estimates
\begin{theorem}
\label{app:thm:moment_estimates}
	Let us assume the same conditions as in Thm.~\ref{app:thm:exist_uniq} and for each $q\in Q$, let $\xi^q_t$ be the unique solution of~\eqref{eq:app:thm:exist_uniq:main_sde}. Let $m\geq 1$ and suppose $\sup_{q\in Q} \E \sup_{t\in[s,T]} | \phi^q_t |^{2m} < \infty$.
	Then, there exists a constant $C>0$ depending only on $L,T,m$ such that
	\begin{align*}
		& \sup_{q \in Q} \E \sup_{t\in[s,T]} | \xi^q_t |^{2m}
		\leq
		C \left(
			1 + \sup_{q\in Q} \E \sup_{t\in[s,T]} | \phi^q_t |^{2m}
		\right).
	\end{align*}
\end{theorem}
\begin{proof}
	We have
	\begin{align*}
		| \xi^{q}_t |^{2m}
		\leq&
		3^{2m-1} | \phi^{q}_t |^{2m}
		+ {(3(t-s))}^{2m-1} \int_{s}^{t}
		|
			B(\xi^{q}_v, v, q)
		|^{2m} dv \\
		&+ 3^{2m-1}
		\left|
			\int_{s}^{t}
				S(\xi^{q}_v, v, q)
			dW_v
		\right|^{2m}
	\end{align*}
	Taking expectations, using It\^{o}'s isometry (inequality version) and Gronwall's inequality, we obtain
	\begin{align*}
		\E | \xi^{q}_t |^{2m}
		\leq&
		3^{2m-1} \sup_{q\in Q} \E \sup_{t\in[s,T]} | \phi^{q}_t |^{2m}
		+ {(3(t-s))}^{2m-1} L^{2m} \int_{s}^{t} (1 + \E | \xi^{q}_v |^{2m}) dv \\
		&+
		3^{2m-1}{(m(2m-1))}^{m} (t-s)^{m-1} L^{2m} \int_{s}^{t} (1 + \E | \xi^{q}_v |^{2m}) dv \\
		\leq&
		C \left(
			1 + \sup_{q\in Q} \E \sup_{t\in[s,T]} | \phi^{q}_t |^{2m}
		\right) < \infty,
	\end{align*}
	with $C$ depending only on $L,T,m$. Next,
	\begin{align*}
		\E \sup_{t\in[s,T]}
		| \xi^{q}_t |^{2m}
		\leq&
		3^{2m-1} \sup_{q\in Q} \E \sup_{t\in[s,T]} | \phi^{q}_t |^{2m} \\
		&+ {3T}^{2m-1} \E
		\int_{s}^{T}
		|
			B(\xi^{q}_v, v, q)
		|^{2m} dv \\
		&+ 3^{2m-1}
		\E \sup_{t\in[s,T]}
		\left|
			\int_{s}^{t}
				S(\xi^{q}_v, v, q)
			dW_v
		\right|^{2m}
	\end{align*}
	Now, in the last term, the stochastic integral is a local martingale and so its absolute value is a submartingale, and hence the last term is bounded by
	\begin{align*}
	\E \sup_{t\in[s,T]}
	\left|
		\int_{s}^{t}
			S(\xi^{q}_v, v, q)
		dW_v
	\right|^{2m}
	\leq&
	\left|
	\int_{s}^{T}
		S(\xi^{q}_v, v, q)
	dW_v
	\right|^{2m} \\
	\leq& C 	\int_{s}^{T}
		| S(\xi^{q}_v, v, q) |^{2m}
	dv.
	\end{align*}
	Thus, using Thm.~\ref{app:thm:moment_estimates} and the linear growth condition, we conclude that
	\begin{align*}
		\sup_{q\in Q} \E \sup_{t\in[s,T]} | \xi^{q}_t |^{2m}
		\leq C
		\left(
			1 + \sup_{q\in Q} \E \sup_{t\in[s,T]} | \phi^q_t |^{2m}
		\right)
	\end{align*}
\end{proof}

% Limit Behavior
Finally, we examine some limiting behavior of solutions $\xi^{q}_t$ as $q\rightarrow q^*$ for some $q^*\in Q$.

\begin{theorem}
\label{app:thm:limit_q}
	Let us assume the same conditions as in Thm.~\ref{app:thm:exist_uniq} and let $q^*\in Q$ be fixed. Suppose further that the following holds for any $t\in[s,T]$, $R>0$ and $\epsilon>0$:
	\begin{enumerate}[(i)]
		\item $\lim_{q\rightarrow q^*}
		\PP \left[
			\sup_{| x | \leq R}
			| B(x,t,q) - B(x,t,q^*) | > \epsilon
		\right] = 0$
		\item $\lim_{q\rightarrow q^*}
		\PP \left[
			\sup_{| x | \leq R}
			| S(x,t,q) - S(x,t,q^*) | > \epsilon
		\right] = 0$
		\item $\lim_{q\rightarrow q^*} \sup_{t\in[s,T]}
		\E | \phi^q_t - \phi^{q^*}_t |^2 = 0$
	\end{enumerate}
	Then, the solutions $\xi^q_t$ of~\eqref{eq:app:thm:exist_uniq:main_sde} satisfy
	\begin{align*}
		\lim_{q\rightarrow q^*} \sup_{t\in [s,T]} \E | \xi^{q}_t - \xi^{q^*}_t |^2 = 0,
	\end{align*}
	i.e. $\xi^q_t \rightarrow \xi^{q^*}_t$ in $\Lcal^2(\Omega)$ uniformly in $t\in[s,T]$.
\end{theorem}
\begin{proof}
	We have
	\begin{align*}
		\xi^{q}_t - \xi^{q^*}_t
		=&
		\zeta^{q}_t
		+ \int_{s}^{t}
			B(\xi^{q}_v, v, q) - B(\xi^{q^*}_v, v, q)
		dv \\
		&+ \int_{s}^{t}
			S(\xi^{q}_v, v, q) - S(\xi^{q^*}_v, v, q)
		dW_v,
	\end{align*}
	where
	\begin{align*}
		\zeta^{q}_t :=&
		\phi^{q}_t - \phi^{q^*}_t
		+ \int_{s}^{t}
		B(\xi^{q^*}_v, v, q) - B(\xi^{q^*}_v, v, q^*)
		dv \\
		&+ \int_{s}^{t}
			S(\xi^{q^*}_v, v, q) - S(\xi^{q^*}_v, v, q^*)
		dW_v.
	\end{align*}
	Using the Lipschitz conditions,
	\begin{align*}
		\E | \xi^{q}_t - \xi^{q^*}_t |^2
		\leq 3 \E | \zeta^{q}_t |^2
		+ 6L^2 \int_{s}^{t} \E | \xi^{q}_v - \xi^{q^*}_v |^2 dv,
	\end{align*}
	which by Gronwall's lemma implies
	\begin{align*}
		\sup_{t\in[s,T]} \E | \xi^{q}_t - \xi^{q^*}_t |^2
		\leq 3e^{6L^2T} \sup_{t\in[s,T]} \E | \zeta^{q}_t |^2.
	\end{align*}
	Thus, it remains to show that $\sup_{t\in[s,T]} \E | \zeta^{q}_t |^2\rightarrow 0$ as $h\rightarrow 0$. Now,
	\begin{align*}
		\sup_{t\in[s,T]}\E | \zeta^{q}_t |^2
		\leq&
		3 \sup_{t\in[s,T]} \E | \phi^{q}_t - \phi^{q^*}_t |^2
		+ 3 T \int_{s}^{T}
			\E | B(\xi^{q^*}_v, v, q) - B(\xi^{q^*}_v, v, q^*) |^2
		dv \\
		&+ 3 \int_{s}^{T}
			\E | S(\xi^{q^*}_v, v, q) - S(\xi^{q^*}_v, v, q^*) |^2
		dv.
	\end{align*}
	For each $v\in[s,T]$, the assumption (i) together with the a.s.\,continuity of $B$ implies $B(\xi^{q^*}_v,v,q) \rightarrow B(\xi^{q^*}_v,v,q^*)$ in probability. Moreover, by Assumption~\ref{app:assu:B_and_S} (ii) the last integrand is bounded by $2L^2(1+ \sup_{v\in[s,T]}| \xi^{q^*}_v |^2)$, which is integrable. By the dominated convergence theorem, the integral vanishes in the limit $h\rightarrow 0$. A similar calculation shows the last integral also vanishes in the same limit. Together with (iii), we arrive at our assertion.
\end{proof}

\section{Derivatives with respect to initial condition}
\label{app:sec:deriv_x}

Let us denote by $\{ X^{x,s,q}_t : t\geq 0 \}$ the stochastic process defined by the SDE
\begin{align}
\label{app:eq:sde_deriv_x}
	dX^{x,s,q}_t &= b(X^{x,s,q}_t, q) dt + \sigma(X^{x,s,q}_t, q) dW_t,
	\quad t \in [s,T], \nonumber \\
	X^{x,s,q}_s &= x.
\end{align}
As in the previous section, $q\in Q$ where $Q$ is a subset of a Euclidean space. Throughout this section, we assume the following:
\begin{assumption}
\label{assu:b_and_sigma}
	The (non-random) functions $b,\sigma$ satisfy
	\begin{enumerate}
		\item Uniform linear growth condition
		\begin{align*}
			| b(x, q) |^2 + | \sigma(x, q) |^2 \leq L^2 ( 1 + |x|^2)
		\end{align*}
		for all $x\in \R^d$, $q\in Q$.
		\item Uniform Lipschitz condition
		\begin{align*}
			| b(x, q) - b(y, q) | + | \sigma(x, q) - \sigma(y, q) | \leq L |x - y|
		\end{align*}
		for all $x, y\in \R^d$, $q\in Q$.
	\end{enumerate}
\end{assumption}

With the above assumptions, by Thm.~\ref{app:thm:exist_uniq} the SDE~\eqref{app:eq:sde_deriv_x} admits a unique solution. The focus of this section is to derive the SDEs that characterize the derivatives of $X^{x,s,q}_t$ with respect to $x$, the initial condition. In doing so, we will make use the results proved in Sec.~\ref{app:sec:general_SDE_thms}.

\begin{definition}
	Let $\Psi:\R^d\rightarrow\R$ and $\psi:\R^d\rightarrow\R^d$ be random functions and suppose for each $i=1,\dots,d$,
	\begin{align*}
		\lim_{h\rightarrow 0}
		\E
		\left| \tfrac{1}{h} [
			\Psi(x_{(1)},\dots,x_{(i-1)},x_{(i)}+h,x_{(i+1)},\dots,x_{(d)}) - \psi_i(x_{(1)},\dots,x_{(d)})
		]
		- \psi_{(i)}(x)
		\right|^2
		= 0.
	\end{align*}
	Then, we call $\psi$ the derivative (in the $\Lcal^2(\Omega)$ sense) of $\Psi$ and write $\partial_{(i)}\Psi = \psi_{(i)}$, or $\nabla \Psi = \psi$. For multidimensional $\Psi$, we similarly define the derivative element-wise. Note that the derivative is almost surely unique, if it exists.
\end{definition}

\begin{lemma}
\label{app:lem:d1x}
	Let $s\in[0,T)$, $q\in Q$ and suppose that $b$ and $\sigma$ are continuously differentiable with respect to $x$. Then, $\nabla X^{x,s,q}_t$ exists and if we write $\xi^{x,s,q}_{(i,j),t} := \partial_{(j)} X^{x,s,q}_{(i),t}$, then it satisfies the linear random-coefficient stochastic differential equation
	\begin{align}
	\label{app:eq:sde_d1x}
		\xi^{x,s,q}_{(i,j),t}
		= \delta_{(i,j)}
		+ \int_{s}^{t} \xi^{x,s,q}_{(k,j),v} \partial_{(k)} {b(X^{x,s,q}_v, v, q)}_{(i)} dv
		+ \int_{s}^{t} \xi^{x,s,q}_{(k,j),v} \partial_{(k)} {\sigma(X^{x,s,q}_v, v, q)}_{(i,l)} dW_{(l),v},
	\end{align}
	where $\delta$ is the usual Kronecker delta.
	Moreover, we have
	\begin{align*}
		\sup_{q\in Q} \E \sup_{t\in[s,T]} | \xi^{x,s,q}_t |^{2m} < \infty
	\end{align*}
	for all $m\geq 1$.
\end{lemma}
\begin{proof}
	Let $j$ be fixed and $h^j$ be a $d$-dimensional vector of $0$'s except the $j^\text{th}$ coordinate where it is equal $h^j_{(j)} = h\neq 0$. Then, we have
	\begin{align*}
		\tfrac{1}{h} (X^{x+h^j,s,q}_{(i),t} - X^{x,s,q}_{(i),t})
		=& \delta_{(i,j)} +
		\tfrac{1}{h} \int_{s}^{t} b(X^{x+h^j,s,q}_v, q)_{(i)} - b(X^{x,s,q}_v, q)_{(i)} dv\\
		&+ \tfrac{1}{h} \int_{s}^{t} \sigma(X^{x+h^j,s,q}_v, q)_{(i,l)} - \sigma(X^{x,s,q}_v, q)_{(i,l)} dW_{(l),v}.
	\end{align*}
	But,
	\begin{align*}
		&\tfrac{1}{h} \int_{s}^{t} b(X^{x+h^j,{(i),s}}_v, q)_{(i)} - b(X^{x,s,q}_{v}, q)_{(i)} dv \\
		=& \int_{0}^{1}
				\int_{s}^{t} \tfrac{1}{h}(X^{x+h^j,s,q}_{(k),v} - X^{x,s,q}_{(k),v})
					\partial_{(k)} {b ( \lambda X^{x+h^j,s,q}_v + (1-\lambda) X^{x,s,q}_v )}_{(i)}
				dv
		d\lambda,
	\end{align*}
	and similarly,
	\begin{align*}
		&\tfrac{1}{h} \int_{s}^{t} {\sigma(X^{x+h^j,s,q}_v, q)}_{(i,l)}
		- {\sigma(X^{x,s,q}_v, q)}_{(i,l)} dW_{(l),v} \\
		=& \int_{0}^{1}
				\int_{s}^{t} \tfrac{1}{h}(X^{x+h^j,s,q}_{(k),v} - X^{x,s,q}_{(k),v})
					\partial_{(k)}
					{\sigma ( \lambda X^{x+h^j,s,q}_v + (1-\lambda) X^{x,s,q}_v )}_{(i,l)}
				dW_{(l),v}
		d\lambda.
	\end{align*}
	Therefore, $\xi^{x,s,q,h}_{t} := \tfrac{1}{h}(X^{x+h^j,s,q}_t - X^{x,s,q}_t)$ satisfies~\eqref{eq:app:thm:exist_uniq:main_sde} with $Q \times [0,1]$ in place of $Q$, $\phi^{q,h}_{t,(i,j)} = \delta_{(i,j)}$ and
	\begin{align*}
		&{B(z,t,q, h)}_{(i)} = z_{(k)}
		\int_{0}^{1} \int_{s}^{t} \partial_{(k)}
		{b ( \lambda X^{x+h^j,s,q}_v + (1 - \lambda) X^{x+h^j,s,q}_v , q)}_{(i)} dv d\lambda, \\
		&{S(z,t,q, h)}_{(i)} = z_{(k)}
		\int_{0}^{1} \int_{s}^{t} \partial_{(k)}
		{\sigma ( \lambda X^{x+h^j,s,q}_v + (1 - \lambda) X^{x+h^j,s,q}_v , q)}_{(i,l)}
		dW_{(l), v} d\lambda,
	\end{align*}
	if $h>0$. If $h<0$ we simply consider $-h$ on the left hand side instead and the proof is identical. Furthermore, the uniform Lipschitz conditions on $b,\sigma$ implies bounded derivatives and so we may apply Thm.~\ref{app:thm:exist_uniq} to conclude that there is a process $\xi^{0,q}_t$ satisfying~\eqref{eq:app:thm:exist_uniq:main_sde} with $h=0$, i.e. satisfies~\eqref{app:eq:sde_d1x}.

	It remains to show $\xi^{q,h}_t \rightarrow \xi^{q, 0}_t$ in $\Lcal^2(\Omega)$ uniformly in $t\in[s,T]$, which amounts to checking conditions (i)-(iii) in Thm.~\ref{app:thm:limit_q}, with $q^* = (q, 0)$. The last condition (iii) is trivially satisfied. As for the first two, it is enough to show that $X^{x+h^j,s,q}_t \rightarrow X^{x,s,q}_t$ in $\Lcal^2(\Omega)$ as $h\rightarrow 0$, uniformly in $x$, which follows from the straightforward estimate
	\begin{align*}
		\E | X^{x,s,q}_t - X^{x+h^j,s,q}_t |^2
		\leq& 3 h^2 + C \int_{s}^{t} \E | X^{x,s,q}_v - X^{x+h^j,s,q}_v |^2 dv \\
		\leq& C'h^2.
	\end{align*}
	Now, we may apply Thm.~\ref{app:thm:limit_q} to deduce the satisfaction of the SDE. Finally, the last moment estimate follows from Thm.~\ref{app:thm:moment_estimates}.
\end{proof}

Let us now extend the above result to higher order derivatives. As before, we denote the order $\alpha$ partial derivative of $\Psi$ in the $\Lcal^2(\Omega)$ sense by
\begin{align*}
	\partial^\alpha_{(J)} \Psi
	\equiv \partial^\alpha_{(j_1,\dots,j_\alpha)} \Psi
\end{align*}
where $J$ is an order $\alpha$ multi-index.

\begin{lemma}
\label{app:lem:d2x}
	Suppose that $b,\sigma \in G^{2}$. Then, for each $i,j_1,j_2\in \{ 1, \dots, d \}$, the derivative
	$\xi^{2,x,s,q}_{(i,j_1,j_2),t} := \partial^2_{(j_1,j_2)} X^{x,s,q}_{(i),t}$ exists and is the unique solution of the linear random-coefficient stochastic differential equation
	\begin{align}
	\label{app:eq:lem:d2x:sde}
		\xi^{2,x,s,q}_{(i,j_1,j_2),t}
		=&
		\int_{s}^{t}
		\partial^2_{(k_1,k_2)} {b(X^{x,s,q}_v, q)}_{(i)}
		\xi^{1,x,s,q}_{(k_1,j_1),v} \xi^{1,x,s,q}_{(k_2,j_2),v} dv \\
		&+ \int_{s}^{t}
		\partial^2_{(k_1,k_2)} {\sigma(X^{x,s,q}_v, q)}_{(i,l)}
		\xi^{1,x,s,q}_{(k_1,j_1),v} \xi^{1,x,s,q}_{(k_2,j_2),v} dW_{(l),v} \\
		&+ \int_{s}^{t}
		\partial_{(k)} {b(X^{x,s,q}_v, q)}_{(i)}
		\xi^{2,x,s,q}_{(k,j_1,j_2),v} dv \\
		&+ \int_{s}^{t}
		\partial_{(k)} {\sigma(X^{x,s,q}_v, q)}_{(i,l)}
		\xi^{2,x,s,q}_{(k,j_1,j_2),v} dW_{(l),v}
	\end{align}
	where $\xi^{1,x,s,q}_{(i,j),t} := \partial_{(j)} X^{x,s,q}_{(i),t}$ is the first derivative. Moreover,
	for each $m\geq 1$, we have $\E \sup_{t\in[s,T]} | \xi^{2,x,s,q}_t |^{2m} \in G$, i.e.
	\begin{align}
	\label{app:eq:lem:d2x:moment}
		\sup_{q \in Q, s\in[0,T]} \E \sup_{t\in[s,T]} | \xi^{2,x,s,q}_t |^{2m}
		\leq \kappa_1 (1 + |x|^{2\kappa_2})
	\end{align}
\end{lemma}
\begin{proof}
	Let us denote
	\begin{align*}
		\phi^{x,s,q}_{(i,j_1,j_2),t} =& \int_{s}^{t}
		\partial^2_{(k_1,k_2)} {b(X^{x,s,q}_v, q)}_{(i)}
		\xi^{1,x,s,q}_{(k_1,j_1),v} \xi^{1,x,s,q}_{(k_2,j_2),v} dv \\
		&+ \int_{s}^{t}
		\partial^2_{(k_1,k_2)} {\sigma(X^{x,s,q}_v, q)}_{(i,l)}
		\xi^{1,x,s,q}_{(k_1,j_1),v} \xi^{1,x,s,q}_{(k_2,j_2),v} dW_{(l),v}.
	\end{align*}
	Note that by Lem.~\ref{app:lem:d1x}, $\E \sup_{t\in[s,T]} | \xi^{1,x,s,q}_t |^{2m}$ is finite for any $m\geq 1$. Then, proceeding as in the proof of Lem.~\ref{app:lem:d1x}, we have
	\begin{align*}
		& \E \sup_{t\in[s,T]} | \phi^{x,s,q}_t |^2 \\
		\leq&
		C \E \sup_{t\in[s,T]}
		\left(
			| \nabla^2 b (X^{x,s,q}_t,q) |^2
			+ | \nabla^2 \sigma (X^{x,s,q}_t,q) |^2
		\right)
		| \xi^{1,x,s,q}_t |^4 \\
		\leq&
		C
		\left[
			\E \sup_{t\in[s,T]}
			\left(
				| \nabla^2 b (X^{x,s,q}_t,q) |^2
				+ | \nabla^2 \sigma (X^{x,s,q}_t,q) |^2
			\right)^2
		\right]^{1/2}
		{[\E \sup_{t\in[s,T]}
		| \xi^{1,x,s,q}_t |^8]}^{\nicefrac{1}{2}} \\
	\end{align*}
	Here, $C$ is independent of $q$ and $s$.
	From the above, using the assumption that $b,\sigma \in G^2$, and the moment estimate in Thm.~\ref{app:thm:moment_estimates} on $X^{x,s,q}_t$, we conclude that
	\begin{align*}
		\sup_{q\in Q, s \in [0,T]} \E \sup_{t\in[s,T]} | \phi^{x,s,q}_t |^2
		\leq \kappa_1 (1 + | x |^{2\kappa})
	\end{align*}
	thus~\eqref{app:eq:lem:d2x:sde} admits a unique solution by Thm.~\ref{app:thm:exist_uniq}, and the solution $\xi^{2,x,s,q}_t$ satisfies the same estimate. Moreover, the estimate above holds for any $2m$ power for $m\geq 1$ by a similar calculation, which shows that $\E \sup_{t\in[s,T]} | \xi^{2,x,s,q}_t |^{2m} \in G$.

	Finally, To show that $\xi^{2,x,s,q}_t$ is the second derivative of $X^{x,s,q}_t$ with respect to $x$, we proceed analogously as in the proof of~\ref{app:lem:d1x}, thanks to estimate~\eqref{app:eq:lem:d2x:moment} and polynomial growth conditions, all the estimates required for interchanging the derivative and the integral signs are satisfied, so the equation for $\xi^{2,x,s,q}_t$ is obtained by formally differentiating under the integral sign with respect to $x$, which is precisely~\eqref{app:eq:lem:d2x:sde}.
\end{proof}

\begin{lemma}
\label{app:lem:dax}
	For each $\alpha\geq 1$, suppose that $b,\sigma \in G^{\alpha+1}$, then, the derivative $\nabla^{\alpha+1} X^{x,s,q}_t$ exists and is the unique a.s.\,continuous solution of the linear random-coefficient SDE
	\begin{align}
	\label{app:eq:lem:dax:sde}
		\xi^{\alpha+1,x,s,q}_{(i,J), t} =& \phi^{x,s,q}_{(i,J), t}
		+ \int_{s}^{t}
		\partial_{(k)} {b(X^{x,s,q}_v, q)}_{(i)}
		\xi^{\alpha+1,x,s,q}_{(k,J), t} dv \\
		&+ \int_{s}^{t}
		\partial_{(k)} {\sigma(X^{x,s,q}_v, q)}_{(i,l)}
		\xi^{\alpha+1,x,s,q}_{(k,J), t} dW_{(l),v},
	\end{align}
	where $J$ is a multi-index of order $\alpha+1$ and $\phi^{x,s,q}_t$ is an a.s.\,continuous stochastic process satisfying $\E \sup_{t\in[s,T]} | \phi^{x,s,q}_t |^{2m} \in G$ for all $m\geq 1$.
	In fact,~\eqref{app:eq:lem:dax:sde} is obtained by formally differentiating~\eqref{app:eq:sde_d1x} under the integral sign $\alpha$ times.
	Moreover, we have $\E \sup_{t\in[s,T]} | \xi^{\alpha+1,x,s,q}_t |^{2m} \in G$ for all $m\geq 1$.
\end{lemma}
\begin{proof}
	The proof is identical to the $\alpha=1$ case in Lem.~\ref{app:lem:d2x}. We omit writing out the whole proof here.
\end{proof}

We now prove the following useful result, which imparts polynomial growth conditions onto expectations functionals.

\begin{proposition}
\label{app:prop:u_poly_estimate}
		Let $s\in[0,T]$ and $g\in G^{\alpha+1}$ for some $\alpha\geq 1$. For $t\in[s,T]$, define
		\begin{align*}
			u(x,s,q,t):=\E g(X^{x,s,q}_t)
		\end{align*}
		Then, $u(\cdot,s,q,t) \in G^{\alpha+1}$ uniformly in $s,q,t$.
\end{proposition}

\begin{proof}
	Consider first the case $\alpha=1$. We shall use the results in Lem.~\ref{app:lem:d1x}-\ref{app:lem:dax} to show that
	\begin{align*}
		\partial_{(i)} u(x,s,q,t) =
		\E \partial_{(k)} g(X^{x,s,q}_t) \partial_{(i)} X^{x,s,q}_{(k),t}
	\end{align*}
	and that $\partial_{(i)} u(x,s,q,t) \in G$. Let $h^j$ be defined as in the proof of~\ref{app:lem:d1x}, we have
	\begin{align*}
		&\tfrac{u(x+h^j,s,q,t) - u(x,s,q,t)}{h} \\
		=&
		\E \int_{0}^{1}
			\tfrac{1}{h}\tfrac{d}{d\lambda} g(\lambda X^{x+h^j,s,q}_t + (1-\lambda) X^{x,s,q}_t)
		d\lambda \\
		=&
		\E
		\int_{0}^{1}
			\partial_{(k)} {g(\lambda X^{x+h^j,s,q}_t + (1-\lambda) X^{x,s,q}_t)}
		d\lambda
		\tfrac{X^{x+h^j,s,q}_{(k),t} - X^{x,s,q}_{(k),t}}{h}.
	\end{align*}
	Now, $\tfrac{1}{h} (X^{x+h^j,s}_{t} - X^{x,s,q}_{t}) \rightarrow \partial_{(j)} X^{x,s,q}_t$ in $\Lcal^2(\Omega)$. Moreover, set
	\begin{align*}
		I_h  :=&
		\int_{0}^{1}
			\partial_{(k)} {g(\lambda X^{x+h^j,s,q}_t + (1-\lambda) X^{x,s,q}_t)}
		d\lambda.
	\end{align*}
	Since $\nabla g$ is continuous, $| I_h - \partial_{(k)} g(X^{x,s,q}_t ) |^2 \rightarrow 0$ in probability. Moreover,
	\begin{align*}
		\E | I_h - \partial_{(k)} g(X^{x,s,q}_t ) |^4 < \infty
	\end{align*}
	by the assumption that $g \in G^1$. Thus, $\{ | I_h - \partial_{(k)} g(X^{x,s,q}_t ) |^2 : h\in[0,1] \}$ is uniformly integrable and so
	$I_h \rightarrow \partial_{(k)} g(X^{x,s,q}_t )$ in $\Lcal^2(\Omega)$.
	We have thus arrived at
	\begin{align*}
		\partial_{(i)} u(x,s,q,t) =
		\E \partial_{(k)} g(X^{x,s,q}_t) \partial_{(i)} X^{x,s,q}_{(k),t},
	\end{align*}
	and in particular,
	\begin{align*}
		| \nabla u(x,s,q,t) |^2
		\leq
		\E | \nabla g(X^{x,s,q}_t) |^2 \E | \nabla X^{x,s,q}_{t} |^2
		\in G,
	\end{align*}
	where we have used Thm.~\ref{app:thm:moment_estimates} and~\ref{app:lem:d1x}. The proof for higher order derivatives follow accordingly by the above procedure, using Lem.~\ref{app:lem:dax}.
\end{proof}

\section{Auxiliary results for the proof of Thm.~\ref{thm:one_step_to_multi_step}}
\label{app:sec:aux_1}

\begin{lemma}
\label{app:lem:delta_prod_estimate}
	Let $\alpha\geq 1$ and suppose $b,\sigma$ satisfy Assumption~\ref{assu:b_and_sigma}. Then, there exists a $K\in G$, independent of $\eta$ and $\epsilon$, such that
	\begin{align*}
		\E \prod_{j=1}^{\alpha+1}
		\left|
			\Dtil_{(i_j)}
		\right|
		\leq K(x) \eta^{\alpha+1}.
	\end{align*}
	where $i_j \in \{ 1,\dots,d \}$ and $C>0$ is independent of $\eta$.
\end{lemma}
\begin{proof}
	We have
	\begin{align*}
		\E |\Dtil(x)|^{\alpha+1}
		\leq&
		2^{\alpha}
		\E \left|
			\int_{0}^{\eta} b(X^{x,0}_s, \eta, \epsilon) ds
		\right|^{\alpha+1}
		+ 2^{\alpha} \eta^{\tfrac{\alpha+1}{2}}
		\E \left|
			\int_{0}^{\eta} \sigma(X^{x,0}_s, \eta, \epsilon) dW_s
		\right|^{\alpha+1} \\
		\leq&
		2^{\alpha} \eta^{\alpha}
		\int_{0}^{\eta}
		\E |b(X^{x,0}_s, \eta, \epsilon)|^{\alpha+1}
		ds
		+ 2^{\alpha} \eta^{\tfrac{\alpha+1}{2}}
		\left|
			\int_{0}^{\eta} \sigma(X^{x,0}_s, \eta, \epsilon) dW_s
		\right|^{\alpha+1}
	\end{align*}
	Using Cauchy-Schwarz inequality, It\^{o}'s isometry, we get
	\begin{align*}
		\E \left|
		\int_{0}^{\eta} \sigma(X^{x,0}_s, \eta, \epsilon) dW_s
		\right|^{\alpha+1}
		\leq&
		{\left(
			\E \left|
			\int_{0}^{\eta} \sigma(X^{x,0}_s, \eta, \epsilon) dW_s
			\right|^{2\alpha+2}
		\right)}^{\nicefrac{1}{2}} \\
		\leq&
		C \eta^{\nicefrac{\alpha}{2}}
		{\left(
			\int_{0}^{\eta}
			\E | \sigma(X^{x,0}_s, \eta, \epsilon) |^{2\alpha + 2}
			ds
		\right)}^{\nicefrac{1}{2}}
	\end{align*}
	where $C$ depends only on $\alpha$. Now, using the linear growth condition (\ref{assu:b_and_sigma} (i)) and the moment estimates in Thm.~\ref{app:thm:moment_estimates}, we obtain the result.
\end{proof}

\begin{lemma}
\label{app:lem:u_eta_estimate}
	Suppose $u \in G^{(\alpha+1)}$ for some $\alpha\geq 1$. Let assumption (i) in Thm.~\ref{thm:one_step_to_multi_step} hold. Then, there exists some $K\in G$, independent of $\eta,\epsilon$, such that
	\begin{align*}
		\left|
			\E u(x^{x,0}_1) - \E u(\Xtil^{x,0}_1)
		\right| \leq K(x) (\eta \rho(\epsilon) + \eta^{\alpha + 1})
	\end{align*}
\end{lemma}
\begin{proof}
	Using Taylor's theorem with the Lagrange form of the remainder, we have
	\begin{align*}
		u(x^{x,0}_1) - u(\Xtil^{x,0}_1)
		=& \sum_{s=1}^{\alpha} \tfrac{1}{s!}
		\sum_{i_1,\dots,i_j=1}^{d}
		\prod_{j=1}^{s} [\Delta_{(i_j)}(x) - \Dtil_{(i_j)}(x)]
		\tfrac{\partial^s u}{\partial x_{(i_1)},\dots x_{(i_j)}}(x) \\
		&+ \tfrac{1}{(\alpha+1)!}
		\sum_{i_1,\dots,i_j=1}^{d}
		\prod_{j=1}^{\alpha+1} [\Delta_{(i_j)}(x) - \Dtil_{(i_j)}(x)] \\
		&\times
		\left[
			\tfrac{\partial^{(\alpha+1)} u}{\partial x_{(i_1)},\dots x_{(i_j)}}
			(x + a \Delta (x))
			- \tfrac{\partial^{(\alpha+1)} u}{\partial x_{(i_1)},\dots x_{(i_j)}}
			(x + \widetilde{ a }\Dtil (x))
		\right],
	\end{align*}
	where $a ,\widetilde{ a }\in[0,1]$. Taking expectations, using assumption (i) of Thm.~\ref{sec:one_to_n_step} and Lem.~\ref{app:lem:delta_prod_estimate}, we get
	\begin{align*}
		| \E u(x^{x,0}_1) - \E u(\Xtil^{x,0}_1) |
		\leq K(x) (\eta\rho(\epsilon) + \eta^{\alpha + 1}).
	\end{align*}
\end{proof}

\section{Auxiliary results for the proof of Thm.~\ref{thm:sme_sgd}}
\label{app:sec:aux_2}

Set in~\eqref{eq:milstein_cts}
\begin{align*}
	b(x,\eta,\epsilon) &= b_0(x, \epsilon) + \eta b_1(x, \epsilon) \\
	\sigma(x,\eta, \epsilon) &= \sigma_0(x, \epsilon),
\end{align*}
We prove the following It\^{o}-Taylor expansion.
\begin{lemma}
\label{app:lem:ito_taylor}
	Let $\psi: \R^d \rightarrow \R$ be a sufficiently smooth function and define the operators
	\begin{align*}
		A_{\epsilon,0}  \psi(x) :=&
		{b_0(x,\epsilon)}_{(i)} \partial_{(i)} \psi(x) \\
		A_{\epsilon,1}  \psi(x) :=&
		{b_1(x,\epsilon)}_{(i)} \partial_{(i)} \psi(x)
		+ \tfrac{1}{2} {\sigma_0(x,\epsilon)}_{(i,k)} {\sigma_0(x,\epsilon)}_{(j,k)}
		\partial^2_{(i,j)} \psi(x) \\
		{[\Lambda_{\epsilon,0} g (x)]}_{(l)} :=&
		{\sigma_0(x,\epsilon)}_{(i,l)} \partial_{(i)} \psi(x),
		\qquad l=1,\dots,d.
	\end{align*}
	Suppose further that $b_0,b_1,\sigma_0 \in G^3$. Then, we have
	\begin{align*}
		\E \psi(X^{x,0}_\eta)
		= \psi(x)
		+ \eta A_{\epsilon,0} \psi(x)
		+ \eta^2 (
			\tfrac{1}{2}A_{\epsilon,0}^2 + A_{\epsilon,1}
		)\psi(x)
		+ \mathcal{O}(\eta^3).
	\end{align*}
\end{lemma}
\begin{proof}
	Using It\^{o}'s formula, we have
	\begin{align*}
		\psi(X^{x,0}_\eta)
		=&
		\psi(x)
		+ \int_{0}^{\eta} A_{\epsilon,0} \psi(X^{x,0}_s) ds
		+ \eta \int_{0}^{\eta} A_{\epsilon,1} \psi(X^{x,0}_s) ds \\
		&+ \sqrt{\eta} \int_{0}^{\eta} \Lambda_{\epsilon,0} \psi(X^{x,0}_s) dW_s
	\end{align*}
	By further application of the above formula to $A_{\epsilon,0}\psi$ and $A_{\epsilon,1}\psi$, we have
	\begin{align*}
		\psi(X^{x,0}_\eta)
		=&
		\psi(x)
		+ \eta A_{\epsilon,0} \psi(x)
		+ \eta^2 (
			\tfrac{1}{2}A_{\epsilon,0}^2 + A_{\epsilon,1}
		)\psi(x) \\
		&+
		\eta \int_{0}^{\eta} \int_{0}^{s}
		(A_{\epsilon,1} A_{\epsilon,0} + A_{\epsilon,0} A_{\epsilon,1}) \psi(X^{x,0}_v) dv ds \\
		&+
		\int_{0}^{\eta} \int_{0}^{s} \int_{0}^{v}
		A_{\epsilon,0}^3 \psi(X^{x,0}_r) dr dv ds \\
		&+
		\eta^2 \int_{0}^{\eta} \int_{0}^{s}
		A_{\epsilon,1}^2 \psi(X^{x,0}_v) dv ds \\
		&+
		\eta \int_{0}^{\eta} \int_{0}^{s} \int_{0}^{v}
		A_{\epsilon,1} A_{\epsilon,0}^2 \psi(X^{x,0}_r) dr dv ds\\
		% stochastic integrals
		&+ \sqrt{\eta} \int_{0}^{\eta}
		\Lambda_{\epsilon,0} \psi (X^{x,0}_s) dW_s \\
		&+ \sqrt{\eta} \int_{0}^{\eta} \int_{0}^{s}
		\Lambda_{\epsilon,0} A_{\epsilon,0} \psi (X^{x,0}_v) dW_v ds \\
		&+ \sqrt{\eta} \int_{0}^{\eta} \int_{0}^{s} \int_{0}^{v}
		\Lambda_{\epsilon,0} A_{\epsilon,0}^2 \psi (X^{x,0}_r) dW_r dv ds\\
		&+ \eta^{\nicefrac{3}{2}} \int_{0}^{\eta} \int_{0}^{s}
		\Lambda_{\epsilon,0} A_{\epsilon,1} \psi (X^{x,0}_v) dW_v ds
	\end{align*}
	Taking expectations of the above, it remains to show that each of the terms in the integral either vanishes, or is $\mathcal{O}(\eta^3)$. This follows immediately from the assumption that $b_0,b_1,\sigma_0\in G^3$ and $\psi \in G^4$. Indeed, observe that all the integrands have at most 3 derivatives in $b_0, b_1, \sigma_0$ and 4 derivatives in $\psi$, which by our assumptions all belong to $G$. Thus, the expectation of each integrand is bounded by $ \kappa_1(1 + \sup_{t\in[0,\eta]} \E | X^{x,0}_t |^{2\kappa_2})$ for some $\kappa_1,\kappa_2$, which by Thm.~\ref{app:thm:moment_estimates} must be finite. Thus, the last 3 stochastic integrals are martingales and their expectation vanish, and the expectations of the other integrals are $\mathcal{O}(\eta^3)$ by the polynomial growth assumption and moment estimates in Thm.~\ref{app:thm:moment_estimates}.
\end{proof}

We also prove a general moment estimate for the generalized SGA iterations~\ref{eq:milstein_discrete}.
\begin{lemma}
\label{app:lem:sgd_moment_estimate}
	Let $\{ x_k : k\geq 0 \}$ be the generalized SGA iterations defined in~\ref{eq:milstein_discrete}. Suppose
	\begin{align*}
		| h(x, \gamma, \eta) | \leq L_\gamma ( 1 + | x |)
	\end{align*}
	for some random variable $L_\gamma > 0$ a.s.\,and $\E {L_\gamma}^{m} < \infty$ for all $m\geq 1$. Then, for fixed $T>0$ and any $m\geq 1$, $\E | x_k |^{m} $ exists and is uniformly bounded in $\eta$ and $k=0,\dots,N \equiv \lfloor T/\eta \rfloor$.
\end{lemma}
\begin{proof}
	For each $k\geq 0$, we have
	\begin{align*}
		|x_{k+1}|^{m} \leq |x_k|^{l}
		+ \sum_{l=1}^{m} \binom{m}{l} |x_k|^{m-l} \eta^{l} |h(x_k,\gamma_k,\eta)|^{m-l}
	\end{align*}
	Now, for $1\leq l\leq m$,
	\begin{align*}
		\E |x_{k}|^{m-l} |h(x_k,\gamma_k,\eta)|^{l}
		=& \E |x_{k}|^{m-l} \E ( |h(x_k,\gamma_k,\eta)|^{l} \big| x_k ) \\
		\leq& \E(L^l_\gamma) \E |x_{k}|^{m-l} (1 + |x_k|^{l}) \\
		\leq& 2 \E(L^l_\gamma) (1 + \E |x_{k}|^{m}).
	\end{align*}
	Hence, if we let $a_k:=\E |x_k|^m$, we have
	\begin{align*}
		a_{k+1} \leq (1 + C \eta) a_k + C' \eta
	\end{align*}
	where $C,C'>0$ are independent of $\eta$ and $k$, which immediately implies
	\begin{align*}
		a_k
		\leq& (a_0 + C'/C) (1+C\eta)^k - C'/C \\
		\leq& (|x_0|^{m} + C'/C) e^{(T/\eta)\log(1 + C\eta)} - C'/C \\
		\leq& (|x_0|^{m} + C'/C) e^{CT} - C'/C.
	\end{align*}
\end{proof}

We also need the following result concerning mollified functions.
\begin{lemma}
\label{app:lem:weak_deriv_poly_growth}
	Let $\epsilon\in(0,1)$ and $\psi$ be continuous with its weak derivative $D \psi$ belonging to $G_w$. Denote by $\psi^\epsilon = \nu^\epsilon * \psi$ the mollification of $\psi$. Then, there exists a $K\in G$ independent of $\epsilon$ such that
	\begin{align*}
		| \psi^\epsilon(x) - \psi(x) | \leq \epsilon K(x)
	\end{align*}
\end{lemma}
\begin{proof}
	We have for almost every $x$,
	\begin{align*}
		| \psi^\epsilon(x) - \psi(x) |
		\leq&
		\int_{\Bcal(0,\epsilon)} \nu^\epsilon(y) |\psi(x-y) - \psi(x)| dy \\
		=&
		\int_{\Bcal(0,\epsilon)} \nu^\epsilon(y)
		\left|
			\int_{0}^{1} D\psi(x-\lambda y) \cdot y d\lambda
		\right|
		dy \\
		\leq& \epsilon
		\int_{\Bcal(0,\epsilon)}
		\int_{0}^{1}
		\nu^\epsilon(y)
		|D\psi(x-\lambda y)| d\lambda
		dy \\
		\leq& \epsilon
		\int_{\Bcal(0,\epsilon)}
		\nu^\epsilon(y)
		\kappa_1 [1 + \kappa_2 (|x| + |y|)] dy \\
		\leq&
		\epsilon K(x).
	\end{align*}
	Since $\psi$ is continuous, the above equality holds for all $x \in \R^d$.
	% Hence, for each $\lambda>0$,
	% \begin{align*}
	% 	\tfrac{1}{\text{Vol}(\Bcal(x,\lambda))}\int_{\Bcal(x,\lambda)}
	% 	| \psi^\epsilon(y) - \psi(y) |
	% \end{align*}
\end{proof}

\section{Auxiliary results for computations in Sec.~\ref{sec:applications}}
\label{app:sec:computation}

\begin{lemma}
	Let $A$ be a real square matrix such that all eigenvalues have positive real part. Then,
	\begin{enumerate}[(i)]
		\item For each $\epsilon >0$, there exists a constant $C_\epsilon > 0$ independent of $t$ but depends on $\epsilon$, such that
		\begin{align*}
			| e^{-t A} | \leq C_\epsilon e^{- t (\min_{i} \Re \lambda_i(A) - \epsilon)}
		\end{align*}
		\item If in addition $A$ is diagonalizable, then there exists a constant $C>0$ independent of $t$ such that
		\begin{align*}
			| e^{-t A} | \leq C e^{- t \min_{i} \Re \lambda_i(A)}
		\end{align*}
	\end{enumerate}
\end{lemma}
\begin{proof}
	(i) We know that $A$ is similar to a Jordan block matrix $J$ so that $e^{-At} = P e^{-Jt} P^{-1}$. Hence, $|e^{-At}| \leq |P||P^{-1}| |e^{-Jt}|$. For each Jordan block $J_k$, we have $J_k = \lambda_k I + N_k$ where $N_k$ is nilpotent ($N_k^k = 0$). Hence,
	\begin{align*}
		e^{- J_k t}
		=& e^{-\lambda_k I t} e^{- N_k t}
		= e^{-\lambda_k t} \sum_{m=0}^{k-1} \tfrac{N_k^m}{m!} (-t)^{m} \\
		=& e^{-( \lambda_k - \epsilon) t}
		\left[
			\sum_{m=0}^{k-1} \tfrac{N_k^m}{m!} (-t)^{m} e^{-\epsilon t}
		\right].
	\end{align*}
	For each $\epsilon>0$ the norm of the last term is uniformly bounded in $t$, and hence we obtain the result.

	(ii) We denote the similarity transformation $A = P D P^{-1}$ where $D$ is the diagonal matrix of eigenvalues of $A$. Defining $Q := P^\dagger P$ ($\dagger$ denotes conjugate transpose), we have
	\begin{align*}
		|e^{-t A}| = \tr(e^{-t A^T} e^{-t A}) = \tr(Q^{-1} e^{-t D^\dagger} Q e^{-t D})
	\end{align*}
\end{proof}

\bibliography{cleaned_ref}

\end{document}